\documentclass{article}
\PassOptionsToPackage{numbers, compress}{natbib}

\usepackage{amsthm}
\usepackage[main,final]{neurips_2026}

\usepackage[utf8]{inputenc} 
\usepackage{newunicodechar}
\newunicodechar{，}{,}
\usepackage[T1]{fontenc}    
\usepackage{hyperref}       
\usepackage{url}            
\usepackage{booktabs}       
\usepackage{amsfonts}       
\usepackage{nicefrac}       
\usepackage{microtype}      
\usepackage{xcolor}         
\usepackage{multirow}   
\usepackage{amsmath}    
\usepackage{subfig}
\usepackage{graphicx}
\usepackage{amssymb}
\usepackage[table]{xcolor}
\usepackage{amsthm}
\usepackage[normalem]{ulem}
\newtheorem{theorem}{Theorem}
\newtheorem{lemma}{Lemma}
\newtheorem{prop}{Proposition}
\newtheorem{corollary}{Corollary}
\newtheorem{assumption}{Assumption}
\usepackage{algorithm}
\usepackage{algpseudocode}
\theoremstyle{definition}
\theoremstyle{plain}
\usepackage{enumitem}

\usepackage{makecell}
\title{StrideDiffusion: Accelerating Diffusion Models for Time-series Generation}

\author{%
\textbf{Du Yin}$^{1,4}$, \textbf{Estrid He}$^{2}$, 
\textbf{Julián Jerónimo Bañuelos}$^{2,3}$, 
\textbf{Yang Yang}$^{1}$, \\
\textbf{Feng Hu}$^{2}$, 
\textbf{Yuchen Luo}$^{2}$, 
\textbf{Hao Xue}$^{1,4}$, 
\textbf{Stephan Sigg}$^{3}$, 
\textbf{Flora Salim}$^{1}$\\
$^1$UNSW\\
$^2$RMIT\\
$^3$Aalto University\\
$^4$HKUST(GZ)\\
\texttt{du.yin@unsw.edu.au, estrid.he@rmit.edu.au}\\
\texttt{haoxue@hkust-gz.edu.cn, flora.salim@unsw.edu.au}
}

\begin{document}

\maketitle

\begin{abstract}
Diffusion models have become competitive generators for time series, but their practical use is limited by the large number of sequential denoising steps required at inference time. Existing fast samplers typically use fixed or generic timestep schedules, overlooking a distinctive property of time-series diffusion: different spectral bands evolve at different rates during the reverse process. We introduce \textbf{StrideDiffusion}, a training-free spectral-aware sampler that adaptively selects the denoising stride from band-level activity. At each step, StrideDiffusion monitors relative \textit{band energy}, \textit{log-power drift}, and \textit{phase velocity} to identify whether high-frequency dynamics remain active or whether the trajectory is dominated by stable low-frequency structure. It then takes fine steps when rapidly varying bands are active and larger jumps once only coarse components remain. A bandwise stability analysis shows that inactive frequency bands change only linearly with the jump size under deterministic affine reverse updates, providing a local justification for spectral activity as a step-size indicator. Across six unconditional time-series generation benchmarks, StrideDiffusion uses only \textbf{14-66} function evaluations instead of \textbf{500/1000} denoising steps, achieving up to \textbf{18.9$\times$} wall-clock speedup while preserving or improving generation quality. On conditional imputation and forecasting, it further delivers \textbf{5-14$\times$} average acceleration with comparable predictive accuracy. These results show that spectral evolution provides a practical and principled signal for fast time-series diffusion sampling. Our code is available at https://anonymous.4open.science/r/stridediff-ts.
\end{abstract}

\section{Introduction}

Diffusion models have become a strong choice for time-series forecasting, imputation, and unconditional generation~\cite{tashiro2021csdi,yuandiffusion,rasul2021autoregressivedenoisingdiffusionmodels}. The main practical bottleneck is inference cost: each sample is produced through hundreds to thousands of sequential denoising steps, with a full forward pass at every step, which precludes low-latency or large-batch deployment even on capable hardware.

Most existing accelerators are designed for image and video diffusion. ODE-based solvers~\cite{karras2022elucidating,lu2022dpm,lu2025dpm} compress the step count uniformly without using anything about the signal being denoised, so they overspend on the easy segments of the trajectory. Feature-caching methods~\cite{ma2024deepcache,li2024faster,liu2025timestep} reuse intermediate activations across timesteps, which works for spatially similar image features but transfers poorly to time series, where every prediction depends on the full temporal context. Distillation~\cite{salimansprogressive,chen2025sana,wang2023videolcm} pushes the step count further at the cost of retraining and a separate student network. None of these methods uses the structure of the reverse process itself to decide where computation matters.

We start from a simple empirical observation (Fig.~\ref{fig1}). In the reverse trajectory of a trained time-series diffusion model, different frequency bands become active at different stages: high-frequency bands carry energy early and then decay, while the low-frequency trend grows monotonically and dominates near the end of sampling. The pattern holds on both stationary and non-stationary data (Section~\ref{sec:redundancy_analysis}), so a uniform schedule overspends on the simple segments. Section~\ref{sec:theory_spectral_stability} makes this precise: under a deterministic single-step affine update, the change of any frequency band is controlled by its energy in $x_\rho$ and the predicted clean signal $\hat{x}_0$, so inactive bands change only $\mathcal{O}(\Delta)$ with the stride. Step size should therefore track band activity rather than be fixed in advance.

\begin{figure}
    \centering
    \includegraphics[width=0.99\linewidth]{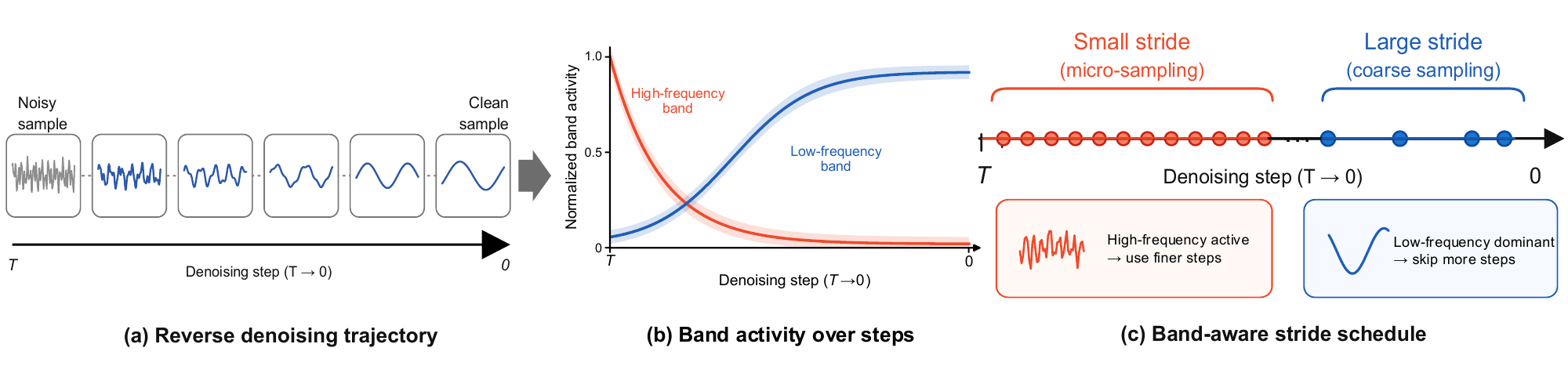}
    \caption{Overview of band-aware sampling. (a) The reverse denoising trajectory from $t{=}T$ to $t{=}0$. (b) High-frequency bands are active early while low-frequency bands dominate near the end. (c) Our scheduler matches the stride to band activity: micro steps when high frequencies are active, large jumps when only the low-frequency trend remains.}
    \vspace{-1.5em}
    \label{fig1}
\end{figure}

Based on this observation, we implement it as a specific sampler named \textbf{StrideDiffusion}. \textbf{StrideDiffusion} is a training-free, band-aware stride scheduler for time-series diffusion. At each step it tracks three spectral statistics: relative band energy, log-power drift, and phase velocity, and gates each band as active or inactive via a power gate combined with a dynamics gate. The active set selects the next stride: fine steps when high-frequency bands are active, mid leaps when only low-frequency bands remain, and coarse leaps once only the DC component remains. Fine steps use deterministic DDIM~\cite{ddim} and coarser strides use DPM-Solver-2 multistep~\cite{lu2022dpm}. DDIM provides the formal affine core of our bandwise stability analysis, while DPM-Solver-2 inherits the same leading structure.

The main contributions are:

\begin{itemize}[
  leftmargin=2em,
  labelsep=0.5em,
  itemsep=1pt,
  topsep=2pt,
  parsep=0pt,
  partopsep=0pt,
  after=\vspace{-1.5em}
]
\item We reveal a coarse-to-fine spectral progression in reverse diffusion across stationary and non-stationary time series (Section~\ref{sec:redundancy_analysis}), motivating non-uniform sampling.
\item We give a bandwise stability result for deterministic DDIM update: inactive bands change only $\mathcal{O}(\Delta)$ with the stride (Section~\ref{sec:theory_spectral_stability}), justifying band activity as a local step-size indicator.
\item We propose StrideDiffusion, a training-free adaptive sampler that uses relative band energy, log-power drift, and phase velocity to adaptively select sampling strides (Section~\ref{sec:adaptive_sample}).
\item On synthetic and real-world time-series datasets, StrideDiffusion preserves generative fidelity at a fraction of the function evaluations of standard solvers (Section~\ref{sec:experiments}).
\end{itemize}

\vspace{-0.8em}
\section{Related Works}
\vspace{-0.8em}

\label{sec:related_work}

\noindent\textbf{Diffusion for Time Series.} Diffusion models have been adapted to time series via score-based and conditional formulations~\cite{tashiro2021csdi,yuandiffusion,rasul2021autoregressivedenoisingdiffusionmodels}, and refined with priors such as non-autoregressive denoising~\cite{shen2023nonautoregressiveconditionaldiffusionmodels}, seasonal-trend decomposition~\cite{yuandiffusion}, retrieval guidance~\cite{liu2024retrieval}, and non-stationary scoring~\cite{ye2025non}.

\textbf{Inference Acceleration.} Training-based accelerators include progressive and consistency distillation~\cite{salimansprogressive,chen2025sana,wang2023videolcm} and quantization~\cite{quan1,quan2,quan3}, while training-free approaches span DDIM~\cite{ddim}, higher-order ODE/SDE solvers~\cite{karras2022elucidating,lu2022dpm,lu2025dpm}, and feature caching~\cite{ma2024deepcache,li2024faster,liu2025timestep}; most target image or video diffusion, and porting them to time series where each step depends on the full temporal context remains open.

\textbf{Spectrum-aware Diffusion.} Diffusion is known to follow a coarse-to-fine spectral trajectory~\cite{rissanen2022generative,yang2023diffusion}, which has motivated samplers that emphasize critical low- or high-frequency steps~\cite{lee2025beta,yuan2025freqprior,yu2025dmfft}; in contrast to these quality-oriented methods, we exploit frequency-band activation as an \emph{efficiency} signal for sampler scheduling. An extended discussion is given in Appendix~\ref{app:related_work_extended}.

\section{Preliminaries}

\subsection{Time Series Generation} 

Let $\mathcal{D}=\{s_{1:L}^{n}\}_{n=1}^{N}$ be a dataset of $N$ multivariate sequences $s_{1:L}\in\mathbb{R}^{L\times d}$, with sequence length $L$ and feature dimension $d$. The goal of time-series generation is to learn a model $G(\cdot)$ whose samples $\hat{s}_{1:L}$ match the data distribution. Under \emph{unconditional} generation, $G$ maps a latent noise $z\sim\mathcal{N}(0,I)$ to a sample, $\hat{s}_{1:L}=G(z)$; under \emph{conditional} generation, $G$ models $p(s_{1:L}\mid c)$ for some conditioning information $c$ (e.g., observed entries for imputation, or a history window for forecasting). The temporal multi-scale structure of time series makes diffusion models a strong fit for this problem, but their iterative sampling is the bottleneck addressed in this work.

\subsection{Diffusion Models for Time Series Generation}
DDPMs~\cite{ho2020denoising} pair a forward chain that corrupts a clean sequence $x_0=s_{1:L}\in\mathbb{R}^{L\times d}$ into Gaussian noise with a learned reverse chain that denoises back. Under a fixed schedule $\{\beta_t\}_{t=1}^{T}$, with $\alpha_t=1-\beta_t$, $\bar{\alpha}_t=\prod_{i\le t}\alpha_i$, and $\rho_t=\bar{\alpha}_t$, the forward marginal admits the closed form
\begin{equation}
\label{eq:forward_reparameterization}
    x_t=\sqrt{\bar{\alpha}_t}\,x_0+\sqrt{1-\bar{\alpha}_t}\,\epsilon,\qquad \epsilon\sim\mathcal{N}(0,I).
\end{equation}

The reverse process predicts either $\hat{x}_0$~\cite{yuandiffusion} or the injected noise $\epsilon_\theta(x_t,t)$~\cite{kongdiffwave,tashiro2021csdi}; we use the latter, from which
\begin{equation}
\label{eq:x0_prediction}
    \hat{x}_0(x_t,t)=\frac{x_t-\sqrt{1-\bar{\alpha}_t}\,\epsilon_\theta(x_t,t)}{\sqrt{\bar{\alpha}_t}}.
\end{equation}

Iterating the reverse step from $x_T$ down to $x_0$ is expensive. DDIM~\cite{ddim} swaps the Markovian step for a non-Markovian one that may jump to any earlier $t_{\mathrm{next}}<t$:
\begin{equation}
\label{eq:ddim}
    x_{t_{\mathrm{next}}}=\sqrt{\bar{\alpha}_{t_{\mathrm{next}}}}\,\hat{x}_0(x_t,t)+\sqrt{1-\bar{\alpha}_{t_{\mathrm{next}}}-\sigma_{t,t_{\mathrm{next}}}^{2}}\,\epsilon_\theta(x_t,t)+\sigma_{t,t_{\mathrm{next}}}\,z,
\end{equation}
with $z\sim\mathcal{N}(0,I)$ and $\sigma_{t,t_{\mathrm{next}}}=\eta\sqrt{(1-\bar{\alpha}_{t_{\mathrm{next}}})/(1-\bar{\alpha}_t)}\sqrt{1-\bar{\alpha}_t/\bar{\alpha}_{t_{\mathrm{next}}}}$; setting $\eta=0$ removes the noise term and yields the deterministic update used throughout this work.

\section{Redundancy in Iterative Denoising}

\begin{figure}[H]
    \centering
    \includegraphics[width=0.79\linewidth]{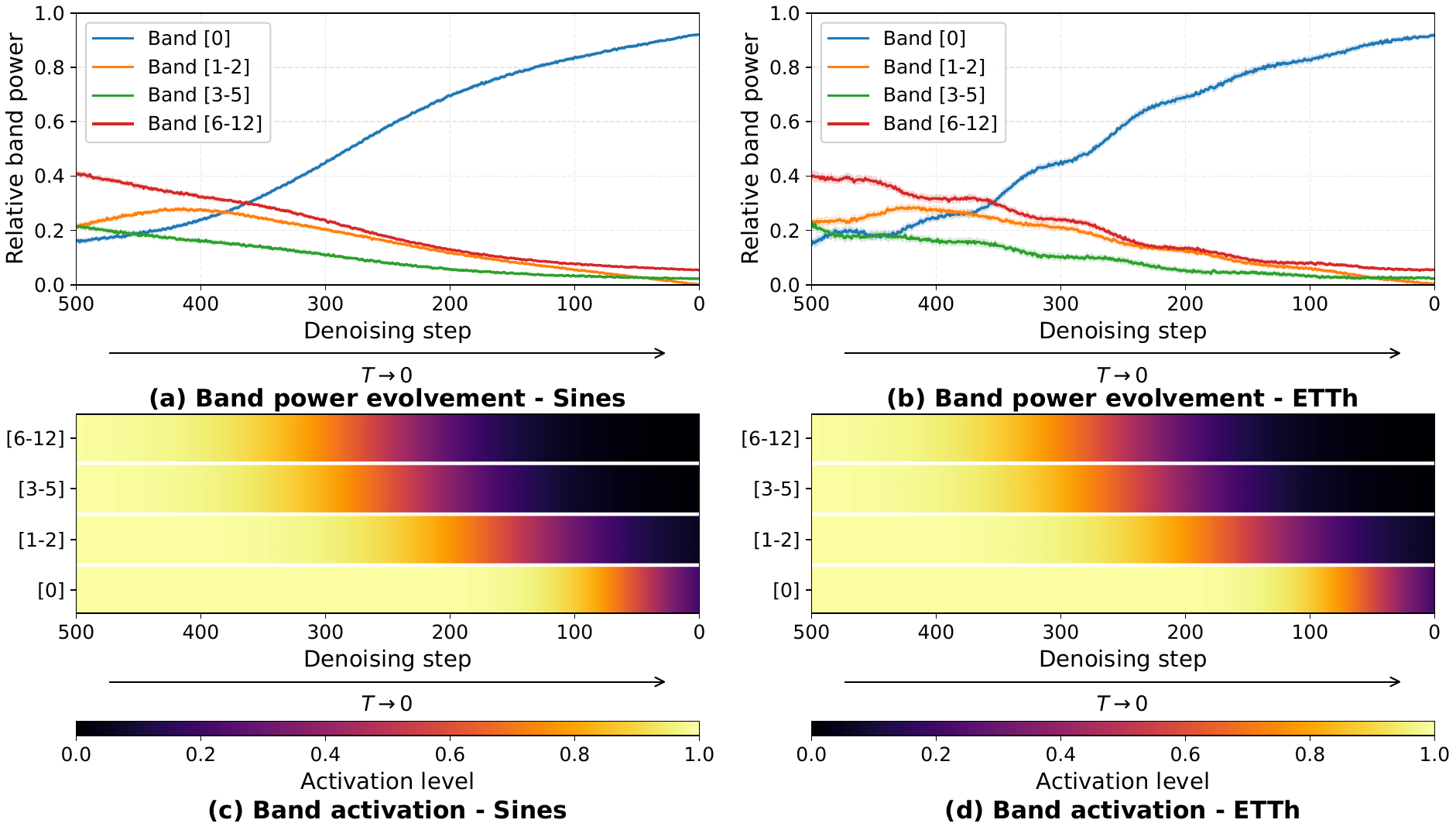}
    \caption{Band activity during reverse denoising on Sines (left) and ETTh (right) with $T{=}500$. Top: relative power of bands $[0]$, $[1\text{-}2]$, $[3\text{-}5]$, $[6\text{-}12]$. Bottom: activation heatmaps. High-frequency bands decay first while the global trend band $[0]$ grows monotonically; the coarse-to-fine progression holds on both stationary (Sines) and non-stationary (ETTh) data.}
    \label{fig:band_activity}
    \vspace{-0.6em}
\end{figure}

\label{sec:redundancy_analysis}
Iterative denoising is widely known to contain redundant steps that can be skipped without quality loss~\cite{salimansprogressive,mao2023leapfrog,ma2024deepcache,wimbauer2024cache}. Below we analyze where this redundancy lives in the time-series setting by tracking how individual frequency bands evolve along the reverse trajectory.

\textit{Observation.}\,\textbf{Band activity evolves coarse-to-fine over the reverse process.} On Sines~\cite{yoon2019time} and ETTh~\cite{haoyietal-informer-2021} (length~24, $T{=}500$), the trace in Fig.~\ref{fig:band_activity} shows: \textbf{(i)} at the outset, energy is broadly distributed across low ($[1\text{-}2]$), mid ($[3\text{-}5]$), and high ($[6\text{-}12]$) bands; \textbf{(ii)} high-frequency bands decay first while the global trend band $[0]$ grows monotonically and dominates late sampling, so high-frequency updates become redundant earlier than low-frequency ones; \textbf{(iii)} the pattern holds on both stationary (Sines) and non-stationary (ETTh) data, with a cleaner hierarchy on Sines.



\subsection{Spectral Stability of Deterministic Single-Step Linear Updates}
\label{sec:theory_spectral_stability}


StrideDiffusion uses DDIM for fine steps and DPM-Solver-2 multistep for leaps. 
We show below that the bandwise change is controlled by band energy in $x_\rho$ and $\hat{x}_0$, so band activity can serve as a local step-size indicator.

\noindent\textbf{Setup.}
Let $x_\rho\in\mathbb{R}^{L\times d}$ denote a sample at noise level $\rho=\bar\alpha_t$ ($L$ sequence length, $d$ channels). For frequency band $B$, write $\mathcal{P}_B = \mathcal{F}^{-1} M_B \mathcal{F}$ for the orthogonal Fourier projector keeping frequencies in $B$, with band energy $E_\rho[B]=\|\mathcal{P}_B x_\rho\|_F^2$ and relative energy $\pi_\rho[B]=E_\rho[B]/\sum_{B'}E_\rho[B']$ (full setup in Appendix~\ref{app:fourier_projectors}).

For a deterministic DDIM update from \(\rho\) to \(\rho'=\rho+\Delta\), the reverse step admits the affine form
\begin{equation}
\label{eq:single_step_affine}
x_{\rho'} =
s_{\rho,\rho'}x_\rho
+
a_{\rho,\rho'}\hat{x}_0(x_\rho,\rho),
\end{equation}
where
\[
s_{\rho,\rho'}=\sqrt{\frac{1-\rho'}{1-\rho}},
\qquad
a_{\rho,\rho'}=
\sqrt{\rho'}
-
\sqrt{\rho}\sqrt{\frac{1-\rho'}{1-\rho}}.
\]
The DPM-Solver-2 branch has the same leading structure plus a multistep history correction, analyzed in Appendix~\ref{app:dpm_solver_extension}.


\begin{lemma}[Bandwise sensitivity of deterministic DDIM]
\label{lem:bandwise_ddim_sensitivity}
Let $\rho'=\rho+\Delta$ with $\Delta>0$. For any frequency band $B$, the deterministic DDIM update in Eq.~\eqref{eq:single_step_affine} satisfies
\begin{equation}
    \mathcal{P}_B(x_{\rho'}-x_\rho)
    =
    (s_{\rho,\rho'}-1)\mathcal{P}_B x_\rho
    +
    a_{\rho,\rho'}\mathcal{P}_B\hat{x}_0(x_\rho,\rho).
    \label{eq:bandwise_exact_update}
\end{equation}
Moreover, for sufficiently small $\Delta$,
\begin{equation}
    \mathcal{P}_B(x_{\rho'}-x_\rho)
    =
    -\frac{\Delta}{2(1-\rho)}
    \mathcal{P}_B x_\rho
    +
    \frac{\Delta}{2\sqrt{\rho}(1-\rho)}
    \mathcal{P}_B\hat{x}_0(x_\rho,\rho)
    +
    O(\Delta^2).
    \label{eq:bandwise_first_order}
\end{equation}
\end{lemma}
For deterministic DDIM, the bandwise change is controlled by the band energy of both \(x_\rho\) and \(\hat{x}_0(x_\rho,\rho)\). 
Thus, inactivity in both quantities is a sufficient condition for local band stability. 
For DPM-Solver-2, the same leading structure appears together with an explicit history correction, analyzed in Appendix~\ref{app:dpm_solver_extension}.

\begin{corollary}[Stability of inactive bands]
\label{cor:inactive_band_stability}
Assume $\rho \in [\rho_{\min},1-\gamma]$ for some $\rho_{\min}>0$ and $\gamma>0$, and let
$\Delta \le \gamma/2$. If a frequency band $B$ satisfies
\begin{equation}
    \|\mathcal{P}_B x_\rho\|_F \le \varepsilon_x,
    \qquad
    \|\mathcal{P}_B\hat{x}_0(x_\rho,\rho)\|_F \le \varepsilon_0,
    \label{eq:inactive_band_condition}
\end{equation}
then the deterministic DDIM update satisfies
\begin{equation}
    \|\mathcal{P}_B(x_{\rho'}-x_\rho)\|_F
    \le
    C_{\rho_{\min},\gamma}
    \Delta
    \bigl(\varepsilon_x+\varepsilon_0\bigr),
    \label{eq:inactive_band_bound}
\end{equation}
where $C_{\rho_{\min},\gamma}$ is a constant depending only on the admissible noise-level range.
\end{corollary}


\textbf{Interpretation.}
Inactive bands are perturbed only by $O(\Delta)$, so larger strides are less risky when high-frequency bands remain low-energy and slowly varying; once a high-frequency band carries non-negligible energy or changes rapidly, the bound becomes loose and finer steps are needed. The constants in Eq.~\eqref{eq:inactive_band_bound} also blow up as $\rho\to 1$, motivating conservative steps near the end of sampling. Together this calls for a band-activity gate keyed on energy and its temporal variation. In practice, StrideDiffusion uses DDIM for fine steps and a second-order multistep solver for larger leaps; the latter introduces a history correction that is analyzed separately in Appendix~\ref{app:dpm_solver_extension}.

\subsection{Spectral-Guided Fast Inference}
StrideDiffusion turns the spectral observation of Section~\ref{sec:redundancy_analysis} into a sampler. A band-gating module monitors which frequency bands are active at each denoising step, and an adaptive scheduler maps the active set to a stride: fine when high-frequency bands are active, coarse when only low-frequency components remain. Figure~\ref{fig:framework} shows the overall design.


\begin{figure*}
\centering
\setlength{\abovecaptionskip}{0pt}
\setlength{\belowcaptionskip}{-2pt}
    \includegraphics[width=0.93\textwidth]{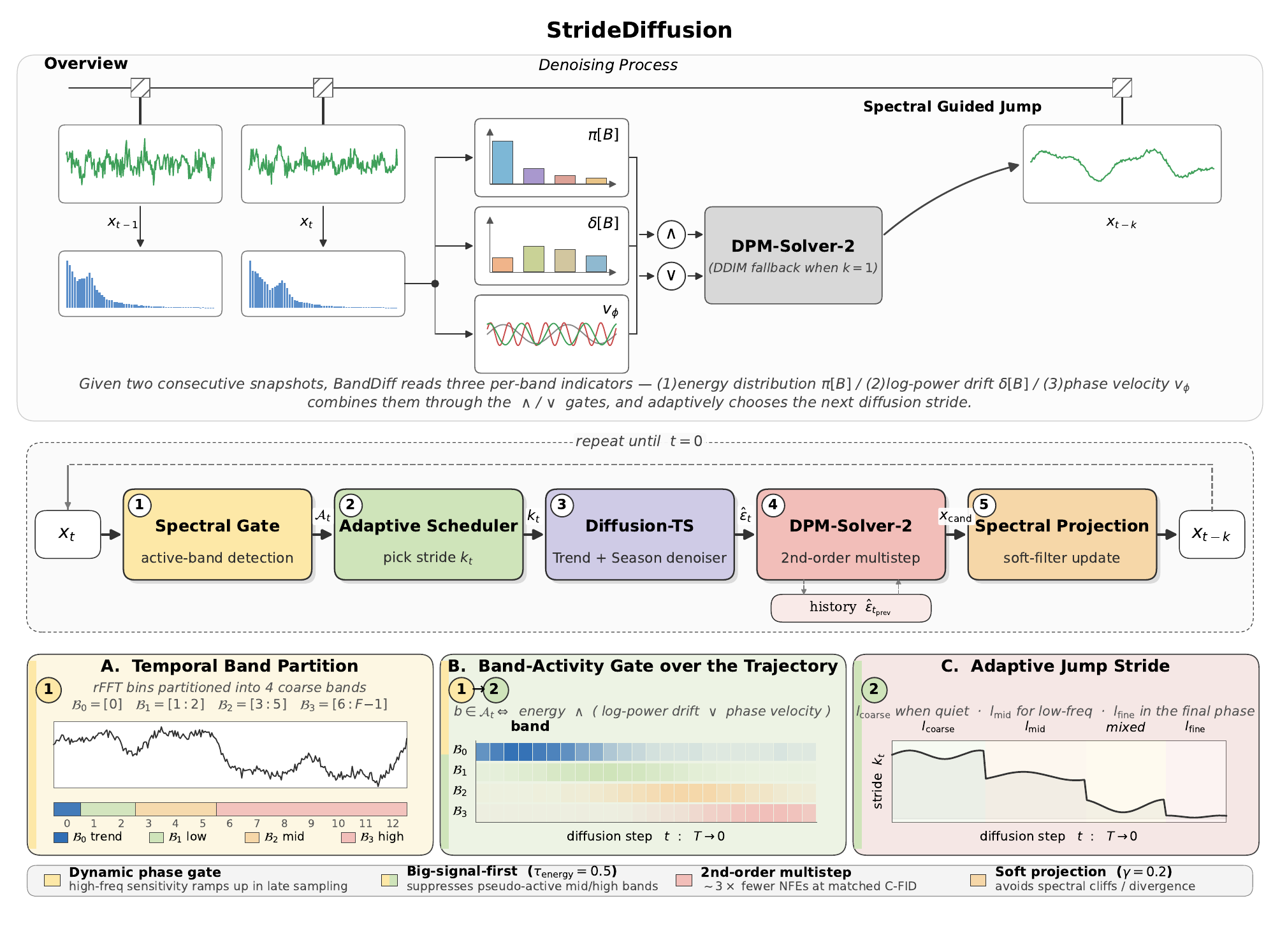}
    \caption{Overview of the proposed \textbf{StrideDiffusion} for spectral-guided fast inference of time series diffusion model. Given two consecutive time series snapshots (in green), StrideDiffusion examines the frequency-domain band activity (in blue) through energy distribution, magnitude drift, and phase velocity, and adaptively determines the next diffusion step. When bands are all active, StrideDiffusion performs fine updates with micro step size; when high-frequency bands become settled, it leaps forward with coarse step size to accelerate inference.}    
    \label{fig:framework}
    \vspace{-0.5em}
\end{figure*}

\subsubsection{Band Gating} 
We propose a spectral-guided gating module to monitor band activity throughout the diffusion process. Given two multivariate time series snapshots at consecutive diffusion steps, $x_{t}$ and $x_{t-1}$, and a partition of the discrete frequency axis into bands, this module identifies the set of bands with active dynamics. Specifically, we map a temporal time series signal $x_t$ to frequency domain using a real FFT:
\begin{equation}
    X_t = \texttt{rFFT}(x_t) \in \mathbb{C}^{F\times d},
\end{equation}
where $F$ represents the total number of frequency components and $d$ represents the feature dimension of the temporal signals. For each band $B\in\mathcal{B}$, we denote its inclusive frequency index set $\mathcal{I}_B \subseteq \{0,\ldots,F-1\}$. For each frequency component, we define its power as
\(
    p_t[f] = \|X_t[f,:]\|_2^2 .
\)
The power within band $B$ can then be computed as:
\begin{equation}
    P_t[B] = \frac{1}{\vert \mathcal{I}_B\vert}\sum_{f\in\mathcal{I}_B} p_t[f].
\end{equation}

\textbf{Energy Fraction.} Let $\mathcal{B}$ represents the set of bands, the energy proportion of each band can be summarized as:
\begin{equation}
    \pi_t[B] = \frac{P_t[B]}{\sum_{B'\in\mathcal{B}} P_t[B']}.
\end{equation}

\textbf{Magnitude Drift.} The magnitude drift over two consecutive denoising steps reflects the evolving trend of the time series samples. We compute the power change over two steps as:
\begin{equation}
    \delta_t[B] = 
    \left|
    \log\!\left(P_t[B]+\varepsilon\right)
    -
    \log\!\left(P_{t-1}[B]+\varepsilon\right)
    \right|,
\end{equation}
where $\varepsilon>0$ is a small stabilizer.

\textbf{Phase Velocity.} In addition to the magnitude drift, the phase change over consecutive denoising steps also indicates if a certain band is active or being activated. We compute the phase change for frequency $f$ from $t\rightarrow t-1$ as:
\begin{equation}
    \phi_t[f] =
    \texttt{angle}
    \left(
    \sum_{j=1}^{d}
    X_{t-1}[f,j]\cdot\overline{X_t[f,j]}
    \right),
\end{equation}
where $\overline{X_t[f,j]}$ represents the conjugate of $X_t[f,j]$. Let
\(\rho_t=\bar{\alpha}_t\),
    and
    \(h_t=\rho_{t-1}-\rho_t\).
Thus, the phase velocity within each band can be computed as a power-weighted average of frequency components within the band:
\begin{equation}
    v_{\phi,t}[B]
    =
    \frac{1}{h_t}
    \sqrt{
    \frac{
    \sum_{f\in\mathcal{I}_B} (p_{t-1}[f]+\varepsilon)\cdot \phi_t[f]^2
    }{
    \sum_{f\in\mathcal{I}_B}(p_{t-1}[f]+\varepsilon)
    }
    } .
\end{equation}

\textbf{Dual Gating.} A frequency band $B$ is determined as active if it meets the criteria imposed by both a \textit{power gate} and a \textit{dynamics gate}. The \textit{power gate} decides a band as active if the band carries a non-negligible proportion of the total signal power. That is, the relative power proportion of the band, $\pi_t[B]$, exceeds the threshold $\tau_{\mathrm{energy}}$. The \textit{dynamics gate} decides a band as active if the band exhibits sufficient temporal variation, expressed as either a magnitude drift $\delta_t[B]$ above the threshold $\tau_{\mathrm{mag}}$, or a phase velocity above the threshold $\tau_{\mathrm{phase}}$. Formally, the dual gating can be expressed as:
\begin{equation}
    \mathcal{A}_t
    =
    \left\{
    B\in\mathcal{B}:
    \pi_t[B] \geqslant \tau_{\mathrm{energy}}
    \land
    \left(
    \delta_t[B] \geqslant \tau_{\mathrm{mag}}
    \lor
    v_{\phi,t}[B] \geqslant \tau_{\mathrm{phase}}
    \right)
    \right\}.
    \label{eq:dual_gate}
\end{equation}
\textbf{Connection to the stability condition.} Corollary~\ref{cor:inactive_band_stability} shows that bandwise stability is guaranteed when a band is inactive in both the current sample \(x_\rho\) and the predicted clean signal \(\hat{x}_0(x_\rho,\rho)\). The gate in Eq.~\eqref{eq:dual_gate} uses trajectory-level statistics from consecutive samples as an efficient proxy for this condition. Incorporating \(\hat{x}_0\)-band energy into the gate is a natural extension, but would change the stride-decision rule; we therefore leave it to future work and discuss this point further in Appendix~\ref{app:discussion_xrho_x0}.

\subsubsection{Sampling with Adaptive Step Size}
\label{sec:adaptive_sample}

The band gating module returns the active set $\mathcal{A}_t$, based on which we adaptively adjust the denoising step size. When only coarse structures remain active, the sampler advances by skipping multiple steps; when high-frequency bands emerge, the schedule switches to micro-step updates. 

Specifically, we define an adaptive leap schedule with three step sizes $\{l_{\text{coarse}}, l_{\text{mid}}, l_{\text{fine}}\}$, each corresponding to a different type of spectral activity. Let $B_{\mathrm{DC}}$ denote the DC band and $\mathcal{B}_{\mathrm{low}}\subseteq\mathcal{B}$ the set of low-frequency bands. The next denoising step is then
\begin{equation}
    t_{\text{next}} = t -
    \begin{cases}
        l_{\text{coarse}}, & \mathcal{A}_t \subseteq \{B_{\mathrm{DC}}\}, \\
        l_{\text{mid}},    & \mathcal{A}_t \subseteq \mathcal{B}_{\mathrm{low}}, \\
        l_{\text{fine}},   & \text{otherwise}.
    \end{cases}
    \label{eq:adaptive_stride}
\end{equation}

\paragraph{Late-step Micro Override.}
The Corollary~\ref{cor:inactive_band_stability} stability bound holds on $\rho\in[\rho_{\min},1-\gamma]$, with constants that grow as $\rho\to 1$, i.e., as sampling approaches the clean-data regime ($t\to 0$). To avoid amplified errors in this regime, we override the band-driven schedule with $l_{\text{fine}}$ whenever $t\le K_{\mathrm{micro}}$, regardless of $\mathcal{A}_t$:
\begin{equation}
    t_{\text{next}} = t - l_{\text{fine}}\quad\text{if}\quad t\le K_{\mathrm{micro}}.
    \label{eq:late_micro}
\end{equation}

The window length $K_{\mathrm{micro}}$ is a hyperparameter; sensitivity to its value is reported in Section~\ref{sec:sensitivity}. The leap schedule is heuristic and a learned policy on the same spectral statistics is left for future work (Appendix~\ref{sec:limitations}).


To transit from $x_t$ to $x_{t_{\text{next}}}$ we apply the deterministic single-step affine update of Section~\ref{sec:theory_spectral_stability} (Eq.~\eqref{eq:single_step_affine}): a DDIM step when the stride equals $l_{\text{fine}}$ or no history is available, and a DPM-Solver-2 multistep update otherwise. 
With \(\Delta_t=\bar{\alpha}_{t_{\mathrm{next}}}-\bar{\alpha}_t>0\), Lemma~\ref{lem:bandwise_ddim_sensitivity} shows that deterministic DDIM perturbs inactive bands only to first order in \(\Delta_t\). 
For DPM-Solver-2, Appendix~\ref{app:dpm_solver_extension} shows that the same leading structure is accompanied by a history correction proportional to the change in the noise predictor. 
The leap-vs-micro discrepancy is analyzed separately in Appendix~\ref{sec:theory_local_error}, under local regularity assumptions on the deterministic update map.

\paragraph{Local-error Reading of The Gate.}
The bandwise stability of Lemma~\ref{lem:bandwise_ddim_sensitivity} extends to a leap-vs-micro-step error bound that is second order in the leap size $H$, with band-restricted constants determined by
the local rate of change of the trajectory and the denoising field (Theorem~\ref{thm:spectral_guided_leap_error}, Appendix~\ref{sec:theory_local_error}). The three
gating quantities used in Section~\ref{sec:adaptive_sample} provide observable finite-difference proxies for local spectral variation of the reverse trajectory (Proposition~\ref{prop:spectral_activity_indicator}). The spectral gate can therefore be read as an adaptive local-error controller: large strides are safe in spectrally stable regions, while fine steps are required once high-frequency dynamics emerge.

\section{Experiments and Results}
\label{sec:experiments}

\subsection{Experiment Setup and Evaluation Protocols}
\label{sec:exp_setup}

\noindent\textbf{Computing Setup and Baselines.} All inference timings are measured on the same machine: 24 vCPUs / 48 threads (Intel Xeon Silver 4310 @ 2.10\,GHz), 256\,GiB RAM, and a single NVIDIA A5000 GPU (24\,GB VRAM). We compare StrideDiffusion against representative diffusion-based generators for time series: \textbf{Diffusion-TS}~\cite{yuandiffusion}, an $\hat{x}_0$-predictor with a Fourier-augmented loss, evaluated in both its vanilla form and its fast-sampler variant (\textbf{Diffusion-TS-fast}); \textbf{DiffWave (Fast)}~\cite{kongdiffwave}, a widely-used diffusion-based waveform model with a fast sampler; and \textbf{DiffTime}, an unconditional CSDI~\cite{coletta2023constrained} variant that pairs diffusion with an ODE sampler.

\noindent\textbf{Datasets and Metrics.} We use four real-world datasets (Stocks, ETTh, Energy, fMRI) and two simulated ones (Sines, MuJoCo). We report four standard metrics: Context-FID for fidelity, Correlational score for temporal structure, Discriminative score for real-vs-synthetic separability, and Predictive score for downstream utility.

\begin{table*}[!h]
\centering
\caption{Unconditional generation across Datasets. \textbf{Bold}: best; \underline{underline}: second best. Sampling times for Diffwave (Fast) and DiffTime are N/A due to prohibitive inference cost.}
\label{tab:results-unconditional-generation}
\footnotesize
\renewcommand\arraystretch{0.79}
\setlength{\tabcolsep}{0.6mm}\begin{tabular}{@{}cccccccc@{}}
\toprule
\textbf{Metric} & \textbf{Method} & \textbf{Sines} & \textbf{Stocks} & \textbf{ETTh} & \textbf{MuJoCo} & \textbf{Energy} & \textbf{fMRI} \\
\midrule
\multirow{6}{*}{\begin{tabular}[c]{@{}c@{}}Context-FID\\Score ($\downarrow$)\end{tabular}}
& Diffusion-TS         & 0.011±.002 & \underline{0.182±.026}    & \textbf{0.136±.009}    & \textbf{0.016±.002}    & 0.102±.011             & \underline{0.109±.005} \\
& Diffusion-TS (Fast)  & 0.016±.001             & 0.195±.032 & 0.227±.005 & 0.031±.003             & \underline{0.091±.006}    & \underline{0.114±.003} \\
& Diffusion-TS (Fast-200) & 0.016±.001          & 0.195±.025 & 0.234±.011             & 0.030±.002             & \underline{0.100±.014} & 0.132±.002 \\
& Diffwave  (Fast)     & 0.014±.002             & 0.232±.032             & 0.873±.061             & 0.393±.041             & 1.031±.131             & 0.244±.018 \\
& DiffTime             & \textbf{0.006±.001}    & 0.236±.074             & 0.299±.044             & 0.188±.028             & 0.279±.045             & 0.340±.015 \\
& Ours                 &  \cellcolor{gray!20} \underline{0.007±.001}              & \cellcolor{gray!20} \textbf{0.110±.022}             & \cellcolor{gray!20} \underline{0.147±.011}             & \cellcolor{gray!20} \underline{0.029±.007} & \cellcolor{gray!20} \textbf{0.087±.034}            & \cellcolor{gray!20} \textbf{0.095±.005} \\
\midrule
\multirow{6}{*}{\begin{tabular}[c]{@{}c@{}}Correlational\\Score ($\downarrow$)\end{tabular}}
& Diffusion-TS         & 0.018±.004 & \underline{0.009±.003} & \textbf{0.054±.003}    & 0.209±.016             & 0.943±.114             & 1.283±.096 \\
& Diffusion-TS (Fast)  & 0.019±.003             & 0.011±.005             & 0.068±.003             & \underline{0.199±.015} & 0.921±.076             & 1.107±.055 \\
& Diffusion-TS (Fast-200) & \underline{0.017±.003} & 0.011±.005             & 0.072±.006             & 0.204±.012             & \underline{0.909±.046} & \underline{1.099±.066} \\
& Diffwave  (Fast)     & 0.022±.005             & 0.030±.020             & 0.175±.006             & 0.579±.018             & 5.001±.154             & 3.927±.049 \\
& DiffTime             & 0.017±.004    & \textbf{0.006±.002}    & \underline{0.067±.005} & 0.218±.031             & 1.158±.095             & 1.501±.048 \\
& Ours                 & \cellcolor{gray!20} \textbf{0.016±.002} & 0.027±.003             & 0.079±.004             & \cellcolor{gray!20} \textbf{0.154±.011}    & \cellcolor{gray!20} \textbf{0.629±.093}    & \cellcolor{gray!20} \textbf{0.982±.046} \\
\midrule
\multirow{6}{*}{\begin{tabular}[c]{@{}c@{}}Discriminative\\Score ($\downarrow$)\end{tabular}}
& Diffusion-TS         & \textbf{0.011±.008}    & \underline{0.080±.018}    & \textbf{0.066±.006}    & \textbf{0.012±.004}    & \underline{0.122±.006}    & 0.153±.012 \\
& Diffusion-TS (Fast)  & 0.018±.006             & 0.109±.014             & 0.095±.006             & \underline{0.027±.006} & \underline{0.142±.007} & \underline{0.116±.012} \\
& Diffusion-TS (Fast-200) & 0.017±.007          & 0.109±.016             & 0.094±.006 & \underline{0.027±.005} & 0.147±.003             & 0.128±.010 \\
& Diffwave  (Fast)     & 0.017±.008             & 0.232±.061             & 0.190±.008             & 0.203±.096             & 0.493±.004             & 0.402±.029 \\
& DiffTime             & \underline{0.013±.006} & \underline{0.097±.016} & 0.100±.007             & 0.154±.045             & 0.445±.004             & 0.245±.051 \\
& Ours                 & 0.017±.004             & \cellcolor{gray!20} \textbf{0.065±.014}             & \cellcolor{gray!20} \underline{0.081±.004}             & \cellcolor{gray!20} \underline{0.027±.004}            & \cellcolor{gray!20} \textbf{0.121±.003}             & \cellcolor{gray!20} \textbf{0.102±.011} \\
\midrule
\multirow{6}{*}{\begin{tabular}[c]{@{}c@{}}Predictive\\Score ($\downarrow$)\end{tabular}}
& Diffusion-TS         & \textbf{0.093±.000}    & \textbf{0.037±.000}    & \textbf{0.121±.002}    & \underline{0.008±.000} & \textbf{0.251±.000} & \underline{0.100±.000} \\
& Diffusion-TS (Fast)  & \underline{0.094±.001} & \textbf{0.037±.000}    & \underline{0.122±.002} & \textbf{0.007±.000}    & \textbf{0.251±.000} & \textbf{0.099±.000} \\
& Diffusion-TS (Fast-200) & \underline{0.094±.000} & \underline{0.038±.000} & \underline{0.122±.002} & \textbf{0.007±.000}  & \textbf{0.251±.000} & \textbf{0.099±.000} \\
& Diffwave  (Fast)     & \textbf{0.093±.000}    & 0.047±.000             & 0.130±.001             & 0.013±.000             & \textbf{0.251±.000} & 0.101±.000 \\
& DiffTime             & \textbf{0.093±.000}    & \underline{0.038±.001} & \textbf{0.121±.004}    & 0.010±.001             & \underline{0.252±.000}             & \underline{0.100±.000} \\
& Ours                 & \cellcolor{gray!20} \underline{0.094±.000} & \cellcolor{gray!20} \textbf{0.037±.000}    & \cellcolor{gray!20} \underline{0.122±.002}             & \cellcolor{gray!20} \textbf{0.007±.001}    & \cellcolor{gray!20} \textbf{0.251±.000}    & \cellcolor{gray!20} \textbf{0.099±.000} \\
\midrule
\multirow{7}{*}{\begin{tabular}[c]{@{}c@{}}Time (s)\\($\downarrow$)\end{tabular}}
& Diffusion-TS      & 84.07 & 40.48 & 179.71 & 214.30 & 696.02 & 461.75 \\
     & Diffusion-TS (Fast) & 86.11 & 40.42 & 180.49 & 217.95 & 687.68 & 458.73 \\
     & Diffusion-TS (Fast-200) & 33.69 & 16.49 & 73.16 & 43.87 & 136.75 & 91.93 \\
& Diffwave  (Fast)     & N/A                     & N/A                    & N/A                    & N/A                    & N/A                    & N/A \\
& DiffTime             & N/A                     & N/A                    & N/A                    & N/A                    & N/A                    & N/A \\
& Ours                 & \cellcolor{blue!10} \textbf{10.23}           & \cellcolor{blue!10}  \textbf{5.09}          & \cellcolor{blue!10} \textbf{9.52}         & \cellcolor{blue!10} \textbf{19.63}         & \cellcolor{blue!10} \textbf{50.56}         & \cellcolor{blue!10} \textbf{38.66} \\
& $\Delta$ Time      & \cellcolor{blue!10} \textbf{8.22 / 3.29}           & \cellcolor{blue!10} \textbf{7.95 / 3.24}         & \cellcolor{blue!10} \textbf{18.88 / 7.68}         & \cellcolor{blue!10} \textbf{10.92 / 2.24}         & \cellcolor{blue!10} \textbf{13.77 / 2.70}         & \cellcolor{blue!10} \textbf{11.94 / 2.38} \\
\bottomrule
\end{tabular}
\end{table*}

\subsection{Unconditional Generation}

At sequence length 24, Table~\ref{tab:results-unconditional-generation} shows StrideDiffusion is up to $18.88\times$ faster and ranks first or second on most metrics; the Pareto plot in Appendix~\ref{sec:pareto_appendix} further confirms it dominates DDIM-$N$ ($N\!\in\!\{10,\ldots,200\}$) on five of six datasets in both Context-FID and wallclock.

\subsection{Conditional Generation}
\label{sec:cond_appendix}

Under the Diffusion-TS protocol with sequence length 48, we evaluate imputation across missing ratios $\{0.1,0.25,0.5,0.75,0.9\}$ (Table~\ref{tab:cond_impute}) and forecasting across horizons $\{6,12,24,36\}$ (Table~\ref{tab:cond_forecast}, Appendix~\ref{sec:cond_appendix}) on Stocks, ETTh, Energy, fMRI. MSE is on par with Diffusion-TS in most cells while inference is $5\text{-}10\times$ faster on average. An ablation with/without Langevin refinement against Diffusion-TS-fast appears in Table~\ref{tab:cond_impute_lv} (Appendix~\ref{sec:cond_appendix}).

\begin{table}[!h]
  \centering
  \caption{Conditional generation results for imputation across missing ratios. Mean across seeds.}
  \label{tab:cond_impute}
  \resizebox{\linewidth}{!}{%
  \renewcommand\arraystretch{0.70}
\begin{tabular}{clcccccccc}
    \toprule
    \multirow{2}{*}{Missing Ratio} & \multirow{2}{*}{Method} & \multicolumn{2}{c}{Stocks} & \multicolumn{2}{c}{ETTh} & \multicolumn{2}{c}{Energy} & \multicolumn{2}{c}{fMRI} \\
    \cmidrule(lr){3-4} \cmidrule(lr){5-6} \cmidrule(lr){7-8} \cmidrule(lr){9-10}
     &  & MSE $\downarrow$ & Time(s) $\downarrow$ & MSE $\downarrow$ & Time(s) $\downarrow$ & MSE $\downarrow$ & Time(s) $\downarrow$ & MSE $\downarrow$ & Time(s) $\downarrow$ \\
    \midrule
    \multirow{2}{*}{0.1} & Diffusion-TS & 0.010 & 56.97 & 0.002 & 214.69 & 0.011 & 716.59 & 0.017 & 436.07 \\
     & Ours  & \textbf{0.008} & \textbf{25.44} & \textbf{0.002} & \textbf{18.91} & \textbf{0.010} & \textbf{120.78} & \textbf{0.015} & \textbf{96.22} \\
    \midrule
    \multirow{2}{*}{0.25} & Diffusion-TS & 0.010 & 55.93 & 0.002 & 210.79 & 0.013 & 704.15 & 0.018 & 418.51 \\
     & Ours & \textbf{0.008} & \textbf{15.47} & 0.003 & \textbf{11.75} & \textbf{0.013} & \textbf{119.17} & \textbf{0.015} & \textbf{92.91} \\
    \midrule
    \multirow{2}{*}{0.5} & Diffusion-TS & 0.011 & 56.24 & 0.003 & 217.75 & 0.016 & 699.99 & 0.020 & 420.23 \\
     & Ours & \textbf{0.008} & \textbf{8.33} & \textbf{0.003} & \textbf{19.06} & \textbf{0.016} & \textbf{121.43} & \textbf{0.017} & \textbf{86.08} \\
    \midrule
    \multirow{2}{*}{0.75} & Diffusion-TS & 0.012 & 56.46 & 0.004 & 219.53 & 0.019 & 681.06 & 0.022 & 420.89 \\
     & Ours & \textbf{0.008} & \textbf{7.30} & 0.005 & \textbf{18.99} & \textbf{0.019} & \textbf{120.73} & \textbf{0.021} & \textbf{78.07} \\
    \midrule
    \multirow{2}{*}{0.9} & Diffusion-TS & 0.012 & 55.44 & 0.004 & 213.46 & 0.019 & 676.07 & 0.022 & 418.91 \\
     & Ours & \textbf{0.008} & \textbf{7.61 }& 0.005 & \textbf{17.65} & \textbf{0.019} & \textbf{129.11} & \textbf{0.021} & \textbf{77.04} \\
 \midrule
\multicolumn{2}{c}{\textbf{Avg. Speedup vs Diffusion-TS}}    
& \multicolumn{2}{c}{\cellcolor{blue!10} \textbf{5.53$\times$}}                   
& \multicolumn{2}{c}{\cellcolor{blue!10} \textbf{12.87$\times$}}             
& \multicolumn{2}{c}{\cellcolor{blue!10} \textbf{5.70$\times$}}             
& \multicolumn{2}{c}{\cellcolor{blue!10} \textbf{4.95$\times$}}             \\
 \bottomrule
  \end{tabular}}
\end{table}


\subsection{Ablation Study}
\label{sec:ablation}

We keep the per-dataset balanced configuration as \textbf{Full} and ablate one component at a time: comparing four variants (Table~\ref{tab:ablation_A}): \emph{Vanilla DDPM} collapses all strides ($l_{\text{coarse}}=l_{\text{mid}}=l_{\text{fine}}=1$); \emph{w/o gate} drops Eq.~\eqref{eq:dual_gate} and forces $l_{\text{coarse}}$ at every step; \emph{w/o late-step micro} sets $K_{\mathrm{micro}}=0$; \emph{w/o $\tau_{\mathrm{energy}}$} disables the power gate. Full delivers $13$-$21\times$ speedup with on-par Context-FID; removing the gate collapses quality (C-FID up to $\sim 38\times$ worse on \textsc{sines}), while removing the energy threshold roughly halves the speedup and still degrades quality on most datasets. Together this identifies the energy-driven gate as the indispensable component.

\begin{table}[!h]
\centering
\caption{Component ablation on unconditional generation. Speedup is relative to vanilla full-$T$ DDPM; \textbf{bold} marks the best value per column (vanilla excluded from speedup comparison). See Appendix~\ref{sec:ablation_appendix} for full discussion.}
\label{tab:ablation_A}
\resizebox{\textwidth}{!}{%
\renewcommand\arraystretch{0.87}
\begin{tabular}{lrrrrrrrr}
\toprule
\multirow{2}{*}{Variant} & \multicolumn{2}{c}{\textsc{sines}} & \multicolumn{2}{c}{\textsc{stocks}} & \multicolumn{2}{c}{\textsc{energy}} & \multicolumn{2}{c}{\textsc{fmri}} \\
\cmidrule(lr){2-3} \cmidrule(lr){4-5} \cmidrule(lr){6-7} \cmidrule(lr){8-9}
 & Speedup $\uparrow$ & C-FID $\downarrow$ & Speedup $\uparrow$ & C-FID $\downarrow$ & Speedup $\uparrow$ & C-FID $\downarrow$ & Speedup $\uparrow$ & C-FID $\downarrow$ \\
\midrule
\textbf{Full (ours)} & 14.0$\times$ & \textbf{0.0086} & 13.2$\times$ & 0.1275 & 20.9$\times$ & 0.0883 & 14.4$\times$ & \textbf{0.0937} \\
Vanilla DDPM (no jump) & 1.0$\times$ & 0.0161 & 1.0$\times$ & 0.2398 & 1.0$\times$ & 0.0817 & 1.0$\times$ & 0.1154 \\
w/o gate & \textbf{41.5$\times$} & 0.3263 & \textbf{29.9$\times$} & 0.4064 & \textbf{42.1$\times$} & 0.1231 & 13.5$\times$ & 0.0937 \\
w/o late-step micro & 21.1$\times$ & 0.0099 & 13.3$\times$ & \textbf{0.1270} & 23.8$\times$ & \textbf{0.0813} & \textbf{16.3$\times$} & 0.1039 \\
w/o $\tau_{\mathrm{energy}}$ & 3.4$\times$ & 0.0103 & 3.0$\times$ & 0.1959 & 10.2$\times$ & 0.0942 & 9.3$\times$ & 0.1548 \\
\bottomrule
\end{tabular}}
\end{table}

\subsection{Hyperparameter Sensitivity Analysis}
\label{sec:sensitivity}

We perform 1-D sweeps on \textsc{sines} around the per-dataset balanced default (Fig.~\ref{fig:hyper}); the same protocol on the other datasets is reported in Appendix~\ref{sec:sensitivity_appendix}. $l_{\mathrm{coarse}}\!\in\!\{10,20,30,50,100\}$ produces a clean U-shape in C-FID with the minimum at the default, confirming a genuine sweet spot. The late-step micro window $K_{\mathrm{micro}}$ shows a threshold-like effect: C-FID is flat for windows $\le 12$ and drops sharply at $20$ (the default), at $\sim 2\times$ sampling cost. In contrast, $\tau_{\mathrm{phase}}$ varied over a $16\times$ range alters C-FID by only $0.024\%$, so gating is effectively dominated by energy and magnitude drift. These patterns hold across datasets: the sampler is robust to $\tau_{\mathrm{phase}}$ and only mildly sensitive to the remaining two within recommended ranges.

\begin{figure}[!t]
\centering
\includegraphics[width=0.81\textwidth]{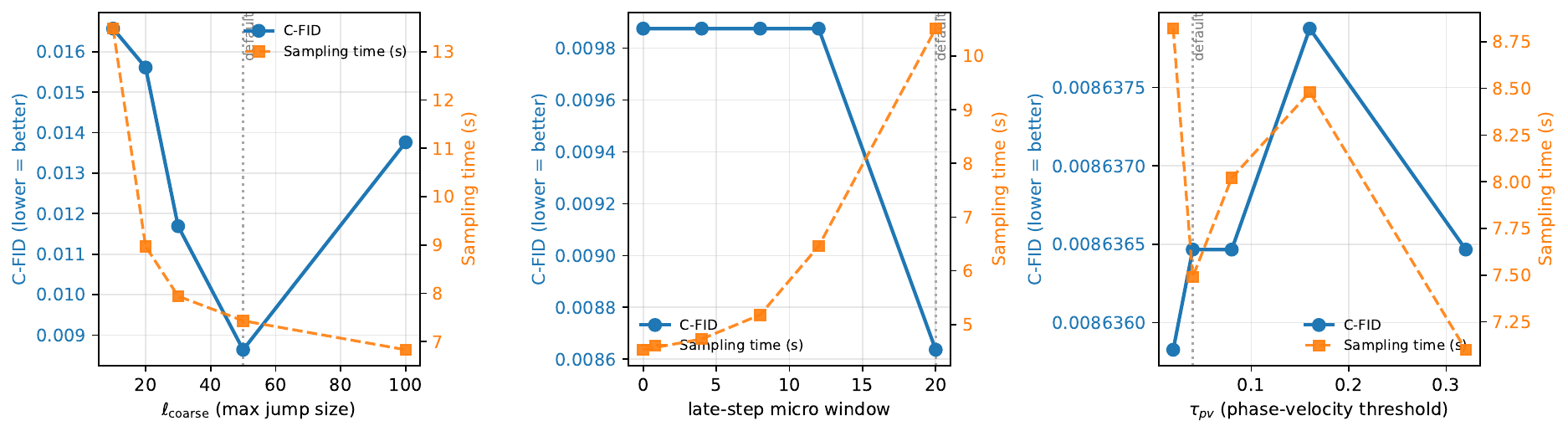}
\caption{Hyperparameter sensitivity on \textsc{sines}. In each row: $l_{\mathrm{coarse}}$ shows a U-shape with the minimum at the default (dotted line), the late-step micro window $K_{\mathrm{micro}}$ has a threshold-like effect, and $\tau_{\mathrm{phase}}$ has no measurable effect ($0.024\%$ C-FID range on \textsc{sines}).}
\label{fig:hyper}

\end{figure}

\subsection{Case Study}
\label{sec:case_study}

\paragraph{Sampler behavior.}
The per-dataset NFE budget ($14$-$66$ vs.\ $T\!\in\!\{500,1000\}$) emerges automatically from the gate (Fig.~\ref{fig:nfe_distribution}): \textsc{etth} stays in \emph{no\_active} throughout ($14.1$ NFE, $\sim 36\times$ compression), while \textsc{sines}/\textsc{energy} open with a \emph{high\_active} phase before joining the \emph{no\_active}$\to$\emph{low\_only}$\to$\emph{late\_micro} tail shared by the rest.

\begin{figure}[!h]
\centering
\includegraphics[width=\linewidth]{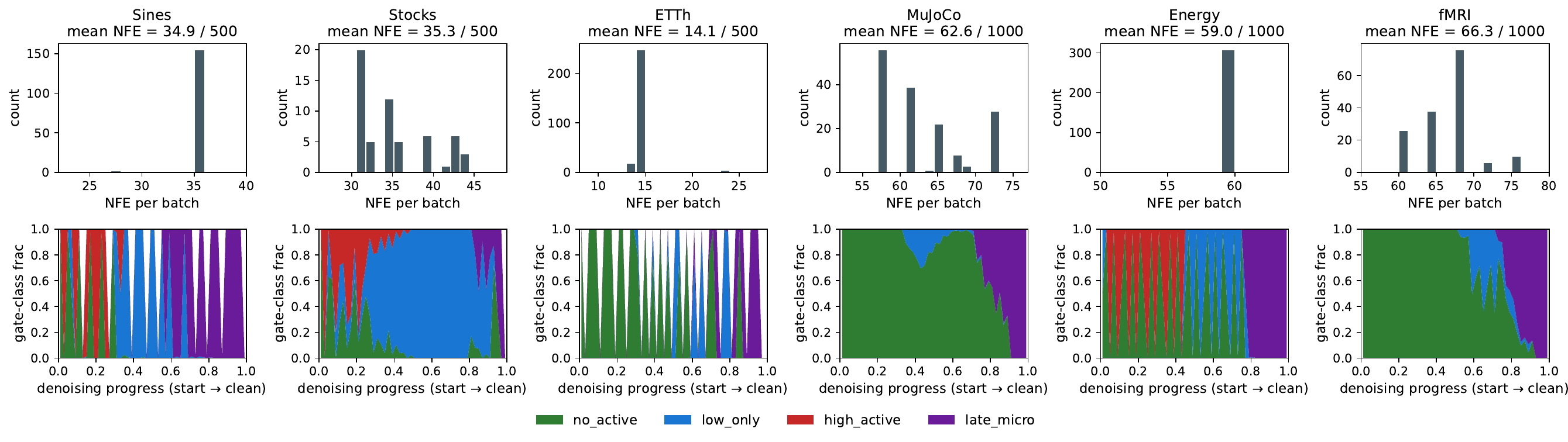}
\caption{Sampler behavior across datasets. \textbf{Top:} per-batch NFE histogram; the dashed vertical line marks the vanilla DDPM budget $T{=}1000$. \textbf{Bottom:} fraction of batches in each gate class (\emph{no\_active} / \emph{low\_only} / \emph{high\_active} / \emph{late\_micro}) at each denoising step ($T{-}t$ on the x-axis).}
\label{fig:nfe_distribution}
\end{figure}

\paragraph{Imputation visualization.}
\label{sec:case_study_qual}
Figure~\ref{fig:cond_qual_stocks} shows the imputation results at missing ratio $0.5$ on a single \textsc{stocks} sequence: our band-aware sampler (red) reconstructs the held-out positions about as accurately as Diffusion-TS DDPM (gray dashed), while running $\sim\!7\times$ faster (Table~\ref{tab:cond_impute}). Forecasting visualization and the full per-(missing-ratio, horizon) sweep across all four datasets are provided in Appendix~\ref{sec:cond_qual_appendix}.

\begin{figure}[!h]
\centering
\includegraphics[width=0.91\textwidth]{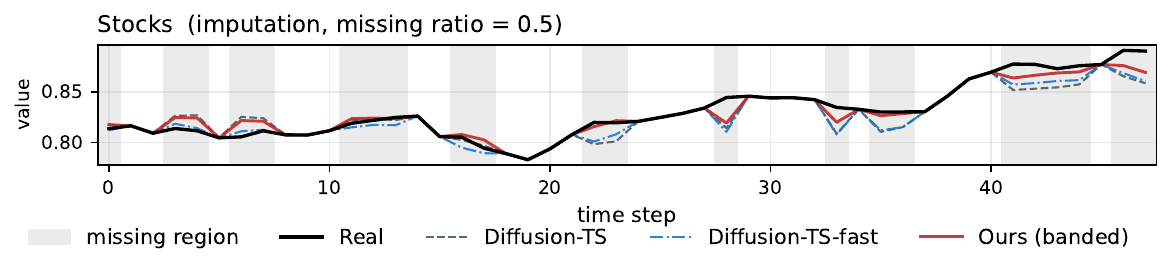}
\caption{Imputation visualization on \textsc{stocks} at missing ratio $0.5$. Black: ground truth. Light gray shading: held-out positions. Gray dashed: full-$T$ Diffusion-TS DDPM. Blue dash-dot: Diffusion-TS-fast (DDIM-200). Red solid: our band-aware sampler.}
\label{fig:cond_qual_stocks}
\end{figure}

\vspace{-0.7em}
\section{Conclusions}

We present StrideDiffusion, a training-free adaptive band-aware sampler for efficient time-series diffusion generation. Motivated by the coarse-to-fine spectral progression we observe in the reverse trajectory, we prove a bandwise stability bound for deterministic single-step affine updates that justifies band activity as a local step-size indicator. Our stability analysis supports this band-aware scheduling, and experiments show that it substantially reduces sampling cost while keeping generation quality across unconditional and conditional tasks. Additionally, learning the stride policy from the same spectral statistics is a natural extension (Appendix~\ref{sec:limitations}).

\begin{ack}
We would like to acknowledge the ARC Centre of Excellence for Automated Decision-Making and Society (CE200100005). We acknowledge the resources and services from the National Computational Infrastructure (NCI), which is supported by the Australian Government. This research is also partially supported by the ARC Training Centre for Whole Life Design of Carbon Neutral Infrastructure (IC230100015).
\end{ack}

\bibliographystyle{plainnat}
\bibliography{reference}







\newpage
\appendix

\section{Extended Related Work}
\label{app:related_work_extended}

\subsection{Diffusion Models for Time Series Generation}
Diffusion models have emerged as a prominent generative paradigm, demonstrating remarkable success in image and speech synthesis. Unlike traditional approaches based on adversarial training or latent-variable inference, diffusion models directly learn the data distribution by reversing a forward noise process. This formulation enables fine-grained control over the generation trajectory, enabling the model to construct sequential outputs through a stepwise refinement.

Such a mechanism has recently been explored and extended in the context of time series data~\cite{tashiro2021csdi, yuandiffusion, rasul2021autoregressivedenoisingdiffusionmodels}, where the iterative denoising process naturally mirrors the sequential structure of temporal data. At each denoising step, the model refines the sequence by conditioning on the surrounding temporal context, typically through attention mechanisms or temporal convolutions. As a result, this stepwise correction enables the model to focus on fine-grained patterns within a local window, while the accumulation of context across iterations gradually integrates long-range temporal dependencies. More recent advancements have further extended the capabilities of diffusion models for time series generation. TimeDiff~\cite{shen2023nonautoregressiveconditionaldiffusionmodels} introduces a non-autoregressive conditional diffusion framework, effectively mitigating exposure bias and enhancing stability in long-horizon forecasting. Diffusion-TS~\cite{yuandiffusion} further integrates seasonal-trend decomposition and Fourier-based objectives to emphasize temporal structure, yielding a more interpretable generative process. RATD~\cite{liu2024retrieval} incorporates retrieval-based guidance into the denoising process, conditioning each step on similar historical trajectories to improve generation accuracy under complex temporal patterns. NsDiff~\cite{ye2025non} addresses non-stationary time series by explicitly modeling distributional shifts through time-conditioned score networks and variance-adaptive noise schedules, enabling the model to adapt to variational dynamics.

\subsection{Inference Acceleration for Diffusion Models}
Recent advancements have highlighted the remarkable generation quality of diffusion models across diverse content creation tasks. Nevertheless, the requirement for numerous iterative denoising steps results in slow inference, limiting their practical deployment.

Previous efforts on accelerating diffusion models often rely on additional resources for optimization or fine-tuning. Knowledge distillation techniques, including progressive distillation and consistency distillation, train student models to achieve comparable quality with fewer sampling steps~\cite{chen2025sana, wang2023videolcm, salimansprogressive}. Quantization methods enhance inference speed via reduced parameter precision, while preserving generation quality by quantization-aware training~\cite{quan1,quan2,quan3}. Distributed inference approaches enable parallel processing across devices to accelerate generation~\cite{dist1}.

In contrast, training-free methods offer immediate applicability to pre-trained models. DDIM reformulates diffusion as deterministic ODE sampling~\cite{ddim}, substantially reducing iterative denoising steps. Additional research has developed sophisticated ODE/SDE solvers that leverage higher-order numerical integration schemes for faster denoising~\cite{karras2022elucidating, lu2022dpm, lu2025dpm}. Recently, another promising direction reduces redundant computation through feature caching across timesteps. DeepCache~\cite{ma2024deepcache} caches high-level UNet features across adjacent timesteps to avoid redundant forward passes. Faster Diffusion~\cite{li2024faster} observes that encoder outputs remain stable during denoising and caches these features while only updating the decoder. TeaCache~\cite{liu2025timestep} leverages timestep embedding similarity to determine optimal caching schedules for attention blocks. While these methods have significantly improved efficiency in image, video, and text generation, their application to the time series diffusion model remains largely unexplored. Unlike visual data, where spatial structures allow for flexible caching strategies, time series exhibit strong temporal dependencies, making feature reuse far more challenging since each prediction depends on the full historical context.

\subsection{Spectrum-aware Diffusion Models}
Recent analyses have revealed a deep connection between diffusion models and spectral processing. \citet{rissanen2022generative} performed spectral analysis, revealing that diffusion models implicitly exhibit coarse-to-fine generation in the frequency domain. \citet{yang2023diffusion} formalized this observation by proving that DDPMs recover low-frequency components first, then progressively add high-frequency details.

Building upon these theoretical findings, subsequent works explicitly incorporate frequency operation to enhance generation quality. \citet{lee2025beta} discovered that substantial low-frequency changes occur early, while high-frequency refinements happen later. Their Beta Sampling method prioritizes these critical steps, achieving superior FID scores compared to uniform sampling. FreqPrior~\cite{yuan2025freqprior} addresses the variance decay problem through frequency-domain noise refinement, significantly improving motion dynamics and imaging details. DMFFT~\cite{yu2025dmfft} modulates frequency amplitude, phase, and bands in U-Net features to improve text-image alignment and structural quality.

Current methods demonstrate that explicit frequency control preserves fine-grained details while maintaining global coherence. However, existing works primarily focus on quality enhancement, with limited exploration of leveraging spectral information for acceleration. In this work, we investigate how frequency band activation patterns can be exploited to improve the inference efficiency of diffusion models.

\section{Spectral Stability of Single-Step Affine Updates}
\label{app:single_step_setup}

\subsection{Fourier Projectors and Band Energy}
\label{app:fourier_projectors}

Let $x_\rho \in \mathbb{R}^{L \times d}$ denote a multivariate time-series sample at noise level
$\rho=\bar{\alpha}_t$, where $L$ is the sequence length and $d$ is the number of channels. We use
$\rho$ instead of the discrete index $t$ to emphasize the dependence on the cumulative signal
coefficient. During reverse sampling, $\rho$ increases from a small value toward $1$.

Let $\mathcal{F}$ denote the discrete Fourier transform along the temporal dimension. For a
frequency band $B$, let $M_B$ be the diagonal mask that keeps frequencies in $B$ and removes
all others. We define the orthogonal Fourier projector
\begin{equation}
    \mathcal{P}_B x
    =
    \mathcal{F}^{-1} M_B \mathcal{F} x .
\end{equation}
The energy of $x_\rho$ in band $B$ is
\begin{equation}
    E_\rho[B]
    =
    \|\mathcal{P}_B x_\rho\|_F^2,
\end{equation}
and the corresponding relative band energy is
\begin{equation}
    \pi_\rho[B]
    =
    \frac{E_\rho[B]}{\sum_{B'}E_\rho[B']}.
\end{equation}

\subsection{Proof of Lemma~\ref{lem:bandwise_ddim_sensitivity}}
\label{app:proof_bandwise_ddim_sensitivity}

The deterministic DDIM update from $\rho$ to $\rho'=\rho+\Delta$ is
\begin{equation}
    x_{\rho'}
    =
    \sqrt{\rho'}\hat{x}_0(x_\rho,\rho)
    +
    \sqrt{1-\rho'}\hat{\epsilon}_\theta(x_\rho,\rho),
\end{equation}
where
\begin{equation}
    \hat{\epsilon}_\theta(x_\rho,\rho)
    =
    \frac{x_\rho-\sqrt{\rho}\hat{x}_0(x_\rho,\rho)}{\sqrt{1-\rho}}.
\end{equation}
Substituting the second equation into the first gives
\begin{align}
    x_{\rho'}
    &=
    \sqrt{\rho'}\hat{x}_0(x_\rho,\rho)
    +
    \sqrt{1-\rho'}
    \frac{x_\rho-\sqrt{\rho}\hat{x}_0(x_\rho,\rho)}{\sqrt{1-\rho}}
    \\
    &=
    \sqrt{\frac{1-\rho'}{1-\rho}}x_\rho
    +
    \left(
    \sqrt{\rho'}
    -
    \sqrt{\rho}\sqrt{\frac{1-\rho'}{1-\rho}}
    \right)
    \hat{x}_0(x_\rho,\rho).
\end{align}
Therefore,
\begin{equation}
    x_{\rho'}
    =
    s_{\rho,\rho'}x_\rho
    +
    a_{\rho,\rho'}\hat{x}_0(x_\rho,\rho),
\end{equation}
where
\begin{equation}
    s_{\rho,\rho'}
    =
    \sqrt{\frac{1-\rho'}{1-\rho}},
    \qquad
    a_{\rho,\rho'}
    =
    \sqrt{\rho'}
    -
    \sqrt{\rho}
    \sqrt{\frac{1-\rho'}{1-\rho}}.
\end{equation}
Subtracting $x_\rho$ from both sides yields
\begin{equation}
    x_{\rho'}-x_\rho
    =
    (s_{\rho,\rho'}-1)x_\rho
    +
    a_{\rho,\rho'}\hat{x}_0(x_\rho,\rho).
\end{equation}
Since $\mathcal{P}_B$ is linear, applying $\mathcal{P}_B$ gives the exact projected update
\begin{equation}
    \mathcal{P}_B(x_{\rho'}-x_\rho)
    =
    (s_{\rho,\rho'}-1)\mathcal{P}_B x_\rho
    +
    a_{\rho,\rho'}\mathcal{P}_B\hat{x}_0(x_\rho,\rho).
\end{equation}

We now derive the first-order expansion. Since $\rho'=\rho+\Delta$,
\begin{equation}
    s_{\rho,\rho'}
    =
    \sqrt{\frac{1-\rho-\Delta}{1-\rho}}
    =
    \sqrt{1-\frac{\Delta}{1-\rho}}.
\end{equation}
Using the Taylor expansion $\sqrt{1-z}=1-\frac{z}{2}+O(z^2)$ gives
\begin{equation}
    s_{\rho,\rho'}-1
    =
    -\frac{\Delta}{2(1-\rho)}
    +
    O(\Delta^2).
\end{equation}
Similarly,
\begin{equation}
    \sqrt{\rho'}
    =
    \sqrt{\rho+\Delta}
    =
    \sqrt{\rho}
    +
    \frac{\Delta}{2\sqrt{\rho}}
    +
    O(\Delta^2).
\end{equation}
Therefore,
\begin{align}
    a_{\rho,\rho'}
    &=
    \sqrt{\rho'}
    -
    \sqrt{\rho}s_{\rho,\rho'}
    \\
    &=
    \left(
    \sqrt{\rho}
    +
    \frac{\Delta}{2\sqrt{\rho}}
    +
    O(\Delta^2)
    \right)
    -
    \sqrt{\rho}
    \left(
    1-\frac{\Delta}{2(1-\rho)}
    +
    O(\Delta^2)
    \right)
    \\
    &=
    \frac{\Delta}{2\sqrt{\rho}}
    +
    \frac{\sqrt{\rho}\Delta}{2(1-\rho)}
    +
    O(\Delta^2)
    \\
    &=
    \frac{\Delta}{2\sqrt{\rho}(1-\rho)}
    +
    O(\Delta^2).
\end{align}
Substituting the expansions of $s_{\rho,\rho'}-1$ and $a_{\rho,\rho'}$ into the exact projected
update gives
\begin{equation}
    \mathcal{P}_B(x_{\rho'}-x_\rho)
    =
    -\frac{\Delta}{2(1-\rho)}
    \mathcal{P}_B x_\rho
    +
    \frac{\Delta}{2\sqrt{\rho}(1-\rho)}
    \mathcal{P}_B\hat{x}_0(x_\rho,\rho)
    +
    O(\Delta^2),
\end{equation}
which proves the lemma.

\subsection{Proof of Corollary~\ref{cor:inactive_band_stability}}
\label{app:proof_inactive_band_stability}

From Lemma~\ref{lem:bandwise_ddim_sensitivity}, we have
\begin{equation}
    \mathcal{P}_B(x_{\rho'}-x_\rho)
    =
    (s_{\rho,\rho'}-1)\mathcal{P}_B x_\rho
    +
    a_{\rho,\rho'}\mathcal{P}_B\hat{x}_0(x_\rho,\rho).
\end{equation}
Taking Frobenius norms and applying the triangle inequality gives
\begin{equation}
    \|\mathcal{P}_B(x_{\rho'}-x_\rho)\|_F
    \le
    |s_{\rho,\rho'}-1|\|\mathcal{P}_B x_\rho\|_F
    +
    |a_{\rho,\rho'}|
    \|\mathcal{P}_B\hat{x}_0(x_\rho,\rho)\|_F.
\end{equation}

We now bound the scalar coefficients. Since $\rho \le 1-\gamma$ and $\Delta\le\gamma/2$,
we have $1-\rho \ge \gamma$ and $1-\rho-\Delta \ge \gamma/2$. Hence the functions
$s_{\rho,\rho'}$ and $a_{\rho,\rho'}$ are smooth on the considered interval. Because
$s_{\rho,\rho}=1$ and $a_{\rho,\rho}=0$, the mean-value theorem implies that there exists a
constant $C_{\rho_{\min},\gamma}>0$, depending only on $\rho_{\min}$ and $\gamma$, such that
\begin{equation}
    |s_{\rho,\rho'}-1|
    \le
    C_{\rho_{\min},\gamma}\Delta,
    \qquad
    |a_{\rho,\rho'}|
    \le
    C_{\rho_{\min},\gamma}\Delta.
\end{equation}
Using the inactivity condition
\begin{equation}
    \|\mathcal{P}_B x_\rho\|_F \le \varepsilon_x,
    \qquad
    \|\mathcal{P}_B\hat{x}_0(x_\rho,\rho)\|_F \le \varepsilon_0,
\end{equation}
we obtain
\begin{equation}
    \|\mathcal{P}_B(x_{\rho'}-x_\rho)\|_F
    \le
    C_{\rho_{\min},\gamma}\Delta
    (\varepsilon_x+\varepsilon_0).
\end{equation}
This proves the desired stability bound.

\subsection{Extension to DPM-Solver-2 Multistep Updates}
\label{app:dpm_solver_extension}

\begin{prop}[Bandwise DPM-Solver-2 update with history correction]
\label{prop:dpm_solver_history}
Let \(\lambda(\rho)=\frac{1}{2}\log\frac{\rho}{1-\rho}\), 
\(h=\lambda(\rho')-\lambda(\rho)\), and
\(r=h_{\mathrm{prev}}/h\). Assume the DPM-Solver-2 multistep update is written in the
\(\epsilon\)-prediction form
\begin{equation}
    x_{\rho'}
    =
    \sqrt{\frac{\rho'}{\rho}}x_\rho
    -
    \sqrt{1-\rho'}(e^h-1)\hat{\epsilon}^{(2)}_\rho,
    \label{eq:dpm_solver2_update}
\end{equation}
where
\begin{equation}
\hat{\epsilon}^{(2)}_\rho
=
\hat{\epsilon}_\rho
+
\frac{1}{2r}
\left(
\hat{\epsilon}_\rho-\hat{\epsilon}_{\mathrm{prev}}
\right),
\qquad
\hat{\epsilon}_\rho
=
\frac{x_\rho-\sqrt{\rho}\hat{x}_0(x_\rho,\rho)}{\sqrt{1-\rho}}.
\end{equation}
Then the update can be written as
\begin{equation}
x_{\rho'}
=
s_{\rho,\rho'}x_\rho
+
a_{\rho,\rho'}\hat{x}_0(x_\rho,\rho)
+
c_{\rho,\rho',r}
\left(
\hat{\epsilon}_\rho-\hat{\epsilon}_{\mathrm{prev}}
\right),
\end{equation}
where
\begin{equation}
s_{\rho,\rho'}
=
\sqrt{\frac{1-\rho'}{1-\rho}},
\qquad
a_{\rho,\rho'}
=
\sqrt{\rho'}
-
\sqrt{\rho}
\sqrt{\frac{1-\rho'}{1-\rho}},
\end{equation}
and
\begin{equation}
c_{\rho,\rho',r}
=
-\frac{\sqrt{1-\rho'}(e^h-1)}{2r}.
\end{equation}
Consequently, for any frequency band \(B\),
\begin{equation}
\mathcal{P}_B(x_{\rho'}-x_\rho)
=
(s_{\rho,\rho'}-1)\mathcal{P}_Bx_\rho
+
a_{\rho,\rho'}\mathcal{P}_B\hat{x}_0(x_\rho,\rho)
+
c_{\rho,\rho',r}
\mathcal{P}_B
\left(
\hat{\epsilon}_\rho-\hat{\epsilon}_{\mathrm{prev}}
\right).
\end{equation}
If \(\rho\in[\rho_{\min},1-\gamma]\), \(\Delta=\rho'-\rho\leq\gamma/2\), and
\(|r|\geq r_{\min}>0\), then
\begin{equation}
|c_{\rho,\rho',r}|
\leq
C_{\rho_{\min},\gamma,r_{\min}}\Delta.
\end{equation}
Therefore, under the inactivity condition in Eq.~\eqref{eq:inactive_band_condition},
\begin{equation}
\|\mathcal{P}_B(x_{\rho'}-x_\rho)\|_F
\leq
C_{\rho_{\min},\gamma}\Delta(\varepsilon_x+\varepsilon_0)
+
C_{\rho_{\min},\gamma,r_{\min}}\Delta
\left\|
\mathcal{P}_B
\left(
\hat{\epsilon}_\rho-\hat{\epsilon}_{\mathrm{prev}}
\right)
\right\|_F.
\end{equation}
\end{prop}

\begin{proof}
Substituting the definition of \(\hat{\epsilon}^{(2)}_\rho\) into
Eq.~\eqref{eq:dpm_solver2_update} gives
\begin{align}
x_{\rho'}
&=
\sqrt{\frac{\rho'}{\rho}}x_\rho
-
\sqrt{1-\rho'}(e^h-1)\hat{\epsilon}_\rho
-
\frac{\sqrt{1-\rho'}(e^h-1)}{2r}
\left(
\hat{\epsilon}_\rho-\hat{\epsilon}_{\mathrm{prev}}
\right).
\end{align}
Using
\[
\hat{\epsilon}_\rho
=
\frac{x_\rho-\sqrt{\rho}\hat{x}_0(x_\rho,\rho)}{\sqrt{1-\rho}},
\]
we obtain
\begin{align}
x_{\rho'}
&=
\left[
\sqrt{\frac{\rho'}{\rho}}
-
\frac{\sqrt{1-\rho'}(e^h-1)}{\sqrt{1-\rho}}
\right]x_\rho
+
\frac{\sqrt{\rho}\sqrt{1-\rho'}(e^h-1)}{\sqrt{1-\rho}}
\hat{x}_0(x_\rho,\rho)
\nonumber\\
&\quad
-
\frac{\sqrt{1-\rho'}(e^h-1)}{2r}
\left(
\hat{\epsilon}_\rho-\hat{\epsilon}_{\mathrm{prev}}
\right).
\end{align}
Since
\[
e^h
=
\exp(\lambda(\rho')-\lambda(\rho))
=
\sqrt{
\frac{\rho'(1-\rho)}
{\rho(1-\rho')}
},
\]
the coefficient of \(x_\rho\) becomes
\begin{align}
\sqrt{\frac{\rho'}{\rho}}
-
\frac{\sqrt{1-\rho'}(e^h-1)}{\sqrt{1-\rho}}
&=
\sqrt{\frac{\rho'}{\rho}}
-
\sqrt{\frac{\rho'}{\rho}}
+
\sqrt{\frac{1-\rho'}{1-\rho}}
\\
&=
\sqrt{\frac{1-\rho'}{1-\rho}}
=
s_{\rho,\rho'}.
\end{align}
Similarly, the coefficient of \(\hat{x}_0(x_\rho,\rho)\) becomes
\begin{align}
\frac{\sqrt{\rho}\sqrt{1-\rho'}(e^h-1)}{\sqrt{1-\rho}}
&=
\sqrt{\rho'}
-
\sqrt{\rho}\sqrt{\frac{1-\rho'}{1-\rho}}
\\
&=
a_{\rho,\rho'}.
\end{align}
Thus,
\[
x_{\rho'}
=
s_{\rho,\rho'}x_\rho
+
a_{\rho,\rho'}\hat{x}_0(x_\rho,\rho)
+
c_{\rho,\rho',r}
\left(
\hat{\epsilon}_\rho-\hat{\epsilon}_{\mathrm{prev}}
\right),
\]
with
\[
c_{\rho,\rho',r}
=
-\frac{\sqrt{1-\rho'}(e^h-1)}{2r}.
\]
Applying the linear projector \(\mathcal{P}_B\) and subtracting
\(\mathcal{P}_Bx_\rho\) gives the bandwise decomposition.

It remains to bound the history coefficient. Since
\[
\lambda'(\rho)
=
\frac{1}{2\rho(1-\rho)},
\]
and \(\rho,\rho'\in[\rho_{\min},1-\gamma/2]\), the mean-value theorem gives
\[
|h|
=
|\lambda(\rho')-\lambda(\rho)|
\leq
C_{\rho_{\min},\gamma}\Delta.
\]
On the same compact interval, \(e^h\) is bounded, and therefore
\[
|e^h-1|
\leq
C_{\rho_{\min},\gamma}|h|
\leq
C_{\rho_{\min},\gamma}\Delta.
\]
Because \(|r|\geq r_{\min}>0\), we have
\[
|c_{\rho,\rho',r}|
=
\frac{\sqrt{1-\rho'}|e^h-1|}{2|r|}
\leq
C_{\rho_{\min},\gamma,r_{\min}}\Delta.
\]
Finally, using the DDIM coefficient bounds
\[
|s_{\rho,\rho'}-1|\leq C_{\rho_{\min},\gamma}\Delta,
\qquad
|a_{\rho,\rho'}|\leq C_{\rho_{\min},\gamma}\Delta,
\]
together with the inactivity conditions
\[
\|\mathcal{P}_Bx_\rho\|_F\leq\varepsilon_x,
\qquad
\|\mathcal{P}_B\hat{x}_0(x_\rho,\rho)\|_F\leq\varepsilon_0,
\]
yields
\[
\|\mathcal{P}_B(x_{\rho'}-x_\rho)\|_F
\leq
C_{\rho_{\min},\gamma}\Delta(\varepsilon_x+\varepsilon_0)
+
C_{\rho_{\min},\gamma,r_{\min}}\Delta
\left\|
\mathcal{P}_B
\left(
\hat{\epsilon}_\rho-\hat{\epsilon}_{\mathrm{prev}}
\right)
\right\|_F.
\]
This proves the result.
\end{proof}

This result shows that the DPM-Solver-2 branch inherits the same DDIM leading-order bandwise structure, but with an additional history correction proportional to the change in the noise predictor. 
When the noise prediction varies smoothly across adjacent solver states, this correction is small; otherwise, it explicitly captures the risk of aggressive multistep extrapolation.

\subsection{Discussion: relation between the stability condition and the practical gate}
\label{app:discussion_xrho_x0}

The bound in Corollary~\ref{cor:inactive_band_stability} depends on both
\(\mathcal{P}_B x_\rho\) and
\(\mathcal{P}_B\hat{x}_0(x_\rho,\rho)\). This gives a sufficient condition for bandwise stability:
if a band is inactive in both the current sample and the predicted clean signal, then a deterministic
DDIM jump cannot substantially alter that band beyond a first-order \(O(\Delta)\) perturbation.

The implemented StrideDiffusion gate does not explicitly test the
\(\hat{x}_0\)-band energy. Instead, it uses observable trajectory statistics computed from consecutive
states \(x_{t-1}\) and \(x_t\), namely relative band energy, log-power drift, and phase velocity. These
quantities provide lightweight proxies for local spectral activity along the reverse trajectory. Thus,
Corollary~\ref{cor:inactive_band_stability} should be interpreted as motivating the use of spectral
activity as a step-size indicator, rather than as proving that every inactive decision made by the
practical gate satisfies the full sufficient condition.

Equivalently, if only the current trajectory band energy is known to be small, the DDIM bound can be
written in the residual form
\[
\|\mathcal{P}_B(x_{\rho'}-x_\rho)\|_F
\leq
C_{\rho_{\min},\gamma}\Delta
\left(
\|\mathcal{P}_B x_\rho\|_F
+
\|\mathcal{P}_B\hat{x}_0(x_\rho,\rho)\|_F
\right).
\]
The practical gate directly controls the first term through spectral energy and local temporal
variation, while the second term remains an ungated predicted-clean residual. This residual explains
why the gate is designed conservatively: high-frequency activity, rapid magnitude drift, or large phase
velocity triggers fine steps, and the late-stage micro-step override further reduces the risk of injecting
unresolved structure near the clean-data regime.




\section{Spectral Activity as a Local Error Indicator}
\label{sec:theory_local_error}

Section~\ref{sec:theory_spectral_stability} shows that inactive bands are stable under any single
deterministic single-step update used by StrideDiffusion (DDIM or DPM-Solver-2 multistep). We now connect this bandwise stability result to the adaptive scheduler.
The goal is to justify why the proposed spectral quantities: relative energy, log-power drift, and phase velocity that can serve as local error indicators for deciding whether a large leap is safe.

\subsection{Bandwise Spectral Dynamics}
Let
\begin{equation}
    X_\rho = \mathcal{F}x_\rho
\end{equation}
be the Fourier representation of the reverse trajectory, and let $X_\rho[B]$ denote the collection
of Fourier coefficients in band $B$. The band power and relative energy are
\begin{equation}
    P_\rho[B] = \|X_\rho[B]\|_F^2,
    \qquad
    \pi_\rho[B] =
    \frac{P_\rho[B]}{\sum_{B'}P_\rho[B']}.
\end{equation}
For two adjacent observations $\rho$ and $\rho+h$, we define the log-power drift
\begin{equation}
    \delta_\rho[B]
    =
    \left|
    \log(P_{\rho+h}[B]+\varepsilon)
    -
    \log(P_\rho[B]+\varepsilon)
    \right|,
    \label{eq:log_power_drift}
\end{equation}
where $\varepsilon>0$ is a small stabilizer. We also define the band-level phase velocity as
\begin{equation}
    v_{\phi,\rho}[B]
    =
    \frac{1}{h}
    \left(
    \frac{
    \sum_{f\in B} w_\rho[f]
    \left|
    \mathrm{wrap}\!\left(
    \angle X_{\rho+h}[f] - \angle X_\rho[f]
    \right)
    \right|^2
    }{
    \sum_{f\in B} w_\rho[f]
    }
    \right)^{1/2},
    \label{eq:phase_velocity}
\end{equation}
where $w_\rho[f]=|X_\rho[f]|^2+\varepsilon$ and $\mathrm{wrap}(\cdot)$ maps phase differences
to $[-\pi,\pi]$.

The following proposition explains why these quantities are useful local indicators.

\begin{prop}[Spectral activity estimates local trajectory variation]
\label{prop:spectral_activity_indicator}
Assume that $X_\rho[f]$ is differentiable in $\rho$ and that $X_\rho[f]\neq 0$ on band $B$.
Then, for small $h$,
\begin{equation}
    \delta_\rho[B]
    =
    h
    \left|
    \frac{d}{d\rho}\log P_\rho[B]
    \right|
    +
    O(h^2),
    \label{eq:log_power_fd}
\end{equation}
and
\begin{equation}
    v_{\phi,\rho}[B]
    =
    \left(
    \frac{
    \sum_{f\in B} w_\rho[f]
    |\partial_\rho \phi_\rho[f]|^2
    }{
    \sum_{f\in B} w_\rho[f]
    }
    \right)^{1/2}
    +
    O(h),
    \label{eq:phase_fd}
\end{equation}
where $\phi_\rho[f]=\angle X_\rho[f]$. Therefore, log-power drift estimates the local radial
change of the band, while phase velocity estimates the local angular change of the band.
\end{prop}

Proposition~\ref{prop:spectral_activity_indicator} shows that the spectral gate is not merely a
heuristic. The energy fraction $\pi_\rho[B]$ identifies whether a band contributes meaningfully to
the current trajectory. The log-power drift $\delta_\rho[B]$ estimates whether the magnitude of the
band is changing rapidly. The phase velocity $v_{\phi,\rho}[B]$ estimates whether the oscillatory
structure of the band is evolving rapidly. A band with low energy, small log-power drift, and small
phase velocity is locally stable and is therefore a candidate for safe skipping.

\paragraph{Leap versus repeated micro-updates.}
We next bound the error introduced by replacing several micro-updates with one larger leap. Let
$\Psi_{\rho,h}$ denote one deterministic single-step update (DDIM or DPM-Solver-2 multistep) update from $\rho$ to $\rho+h$:
\begin{equation}
    x_{\rho+h} = \Psi_{\rho,h}(x_\rho).
\end{equation}
Given a total jump size $H=mh$, define the micro-step trajectory as
\begin{equation}
    x^{\mathrm{micro}}_{j+1}
    =
    \Psi_{\rho+jh,h}(x^{\mathrm{micro}}_j),
    \qquad
    j=0,\ldots,m-1,
\end{equation}
with $x^{\mathrm{micro}}_0=x_\rho$. The corresponding one-leap approximation is
\begin{equation}
    x^{\mathrm{leap}}_{\rho+H}
    =
    \Psi_{\rho,H}(x_\rho).
\end{equation}

\begin{assumption}[Local regularity of the reverse update]
\label{ass:local_regularity}
On the interval $[\rho,\rho+H]$, the deterministic single-step map admits the local expansion
\begin{equation}
    \Psi_{\rho,h}(x)
    =
    x + h f(x,\rho) + R(x,\rho,h),
    \qquad
    \|R(x,\rho,h)\|_F \le C_R h^2,
    \label{eq:ddim_local_expansion}
\end{equation}
where $f$ is locally Lipschitz in $x$ and $\rho$:
\begin{equation}
    \|f(x,\rho)-f(y,\rho)\|_F \le L_x\|x-y\|_F,
    \qquad
    \|f(x,\rho)-f(x,\rho')\|_F \le L_\rho|\rho-\rho'|.
\end{equation}
We also assume $\|f(x,\rho)\|_F\le M$ along the local trajectory.
\end{assumption}

For multistep solvers such as DPM-Solver-2, \(\Psi_{\rho,h}\) should be interpreted as a deterministic map on the augmented solver state, including the stored previous noise prediction; for notational simplicity we write only its action on \(x_\rho\).

\begin{theorem}[Error of a spectral-guided leap]
\label{thm:spectral_guided_leap_error}
Under Assumption~\ref{ass:local_regularity}, the discrepancy between one large deterministic leap of size $H$ (realized by either branch of $\Psi$) and $m$ repeated micro-updates of size $h=H/m$ satisfies
\begin{equation}
    \left\|
    x^{\mathrm{micro}}_{\rho+H}
    -
    x^{\mathrm{leap}}_{\rho+H}
    \right\|_F
    \le
    C H^2
    \left(
    L_x M + L_\rho + C_R
    \right),
    \label{eq:leap_error_bound}
\end{equation}
where $C$ is a universal constant for sufficiently small $H$.

Moreover, for any frequency band $B$,
\begin{equation}
    \left\|
    \mathcal{P}_B
    \left(
    x^{\mathrm{micro}}_{\rho+H}
    -
    x^{\mathrm{leap}}_{\rho+H}
    \right)
    \right\|_F
    \le
    C H^2
    \left(
    L_{x,B}M_B + L_{\rho,B} + C_{R,B}
    \right),
    \label{eq:bandwise_leap_error_bound}
\end{equation}
where $L_{x,B}$, $L_{\rho,B}$, $M_B$, and $C_{R,B}$ are the corresponding band-restricted
regularity constants.
\end{theorem}

\paragraph{Implication for adaptive scheduling.}
Theorem~\ref{thm:spectral_guided_leap_error} states that the error of collapsing micro-updates into
one leap is second order in the leap size, with constants determined by how quickly the reverse
trajectory and denoising field vary. Proposition~\ref{prop:spectral_activity_indicator} gives
observable finite-difference proxies for this variation. Therefore, a large stride is safe when
high-frequency bands have low relative energy, small log-power drift, and small phase velocity.
In this regime, the band-restricted constants in Eq.~\eqref{eq:bandwise_leap_error_bound} are small,
so the leap error remains controlled.

Conversely, when a high-frequency band has substantial energy or rapidly changing magnitude or
phase, the corresponding local variation constants become large. In this case, the scheduler should
avoid a large jump and use fine-grained updates. Thus, the spectral gate can be interpreted as an
adaptive local-error controller: it takes coarse steps in spectrally stable regions and reverts to small
steps when high-frequency dynamics remain active.


\subsection{Proof of Proposition~\ref{prop:spectral_activity_indicator}}
\label{app:proof_spectral_activity_indicator}

We first consider the log-power drift. By definition,
\begin{equation}
    P_\rho[B] = \sum_{f\in B}|X_\rho[f]|^2.
\end{equation}
Assuming differentiability of $X_\rho[f]$ in $\rho$, $P_\rho[B]$ is also differentiable. For small
$h$, Taylor expansion gives
\begin{equation}
    \log(P_{\rho+h}[B]+\varepsilon)
    =
    \log(P_\rho[B]+\varepsilon)
    +
    h
    \frac{d}{d\rho}
    \log(P_\rho[B]+\varepsilon)
    +
    O(h^2).
\end{equation}
Taking the absolute difference yields
\begin{equation}
    \delta_\rho[B]
    =
    h
    \left|
    \frac{d}{d\rho}
    \log(P_\rho[B]+\varepsilon)
    \right|
    +
    O(h^2).
\end{equation}
When $P_\rho[B]\gg\varepsilon$, this reduces to Eq.~\eqref{eq:log_power_fd}. Hence log-power
drift is a finite-difference estimator of the local relative change in band power.

We next consider phase velocity. Write each nonzero Fourier coefficient in polar form:
\begin{equation}
    X_\rho[f] = r_\rho[f]e^{i\phi_\rho[f]}.
\end{equation}
For sufficiently small $h$ and away from phase wrapping discontinuities,
\begin{equation}
    \mathrm{wrap}\!\left(
    \angle X_{\rho+h}[f] - \angle X_\rho[f]
    \right)
    =
    h \partial_\rho \phi_\rho[f] + O(h^2).
\end{equation}
Substituting this expansion into the definition of $v_{\phi,\rho}[B]$ gives
\begin{align}
    v_{\phi,\rho}[B]
    &=
    \frac{1}{h}
    \left(
    \frac{
    \sum_{f\in B}w_\rho[f]
    |h\partial_\rho\phi_\rho[f]+O(h^2)|^2
    }{
    \sum_{f\in B}w_\rho[f]
    }
    \right)^{1/2}
    \\
    &=
    \left(
    \frac{
    \sum_{f\in B}w_\rho[f]
    |\partial_\rho\phi_\rho[f]|^2
    }{
    \sum_{f\in B}w_\rho[f]
    }
    \right)^{1/2}
    +
    O(h).
\end{align}
This proves Eq.~\eqref{eq:phase_fd}. Therefore, log-power drift and phase velocity estimate the
radial and angular components of local spectral variation, respectively.

\subsection{Proof of Theorem~\ref{thm:spectral_guided_leap_error}}
\label{app:proof_spectral_guided_leap_error}

We prove the global bound; the bandwise bound follows by applying the same argument after
projecting all quantities with $\mathcal{P}_B$.

Let $x_0=x_\rho$ and let $x_j=x^{\mathrm{micro}}_j$. By Assumption~\ref{ass:local_regularity},
one micro-step satisfies
\begin{equation}
    x_{j+1}
    =
    x_j
    +
    h f(x_j,\rho+jh)
    +
    R(x_j,\rho+jh,h),
    \qquad
    \|R(x_j,\rho+jh,h)\|_F \le C_Rh^2.
\end{equation}
Summing over $j=0,\ldots,m-1$ gives
\begin{equation}
    x^{\mathrm{micro}}_{\rho+H}
    =
    x_0
    +
    h\sum_{j=0}^{m-1} f(x_j,\rho+jh)
    +
    \sum_{j=0}^{m-1}R(x_j,\rho+jh,h).
    \label{eq:micro_sum}
\end{equation}
Since $mh=H$, the accumulated remainder satisfies
\begin{equation}
    \left\|
    \sum_{j=0}^{m-1}R(x_j,\rho+jh,h)
    \right\|_F
    \le
    m C_R h^2
    =
    C_R Hh
    \le
    C_R H^2,
    \label{eq:micro_remainder}
\end{equation}
where the last inequality holds because $h\le H$.

We next compare the sum of vector fields with $Hf(x_0,\rho)$. For each $j$,
\begin{align}
    \|f(x_j,\rho+jh)-f(x_0,\rho)\|_F
    &\le
    L_x\|x_j-x_0\|_F
    +
    L_\rho jh.
\end{align}
Since $\|f(x,\rho)\|_F\le M$ and the one-step remainder is $O(h^2)$, for sufficiently small $H$,
\begin{equation}
    \|x_j-x_0\|_F
    \le
    C jh M.
\end{equation}
Therefore,
\begin{align}
    h\sum_{j=0}^{m-1}
    \|f(x_j,\rho+jh)-f(x_0,\rho)\|_F
    &\le
    h\sum_{j=0}^{m-1}
    \left(
    C L_x jh M
    +
    L_\rho jh
    \right)
    \\
    &\le
    C H^2(L_xM+L_\rho).
    \label{eq:field_variation_bound}
\end{align}
Combining Eqs.~\eqref{eq:micro_sum}, \eqref{eq:micro_remainder}, and
\eqref{eq:field_variation_bound}, we obtain
\begin{equation}
    x^{\mathrm{micro}}_{\rho+H}
    =
    x_0
    +
    H f(x_0,\rho)
    +
    O\!\left(
    H^2(L_xM+L_\rho+C_R)
    \right).
    \label{eq:micro_expansion}
\end{equation}

The one-leap update also satisfies the local expansion in Assumption~\ref{ass:local_regularity}:
\begin{equation}
    x^{\mathrm{leap}}_{\rho+H}
    =
    \Psi_{\rho,H}(x_0)
    =
    x_0
    +
    H f(x_0,\rho)
    +
    R(x_0,\rho,H),
\end{equation}
with
\begin{equation}
    \|R(x_0,\rho,H)\|_F \le C_RH^2.
    \label{eq:leap_expansion}
\end{equation}
Subtracting Eq.~\eqref{eq:leap_expansion} from Eq.~\eqref{eq:micro_expansion} gives
\begin{equation}
    \left\|
    x^{\mathrm{micro}}_{\rho+H}
    -
    x^{\mathrm{leap}}_{\rho+H}
    \right\|_F
    \le
    C H^2
    \left(
    L_x M + L_\rho + C_R
    \right),
\end{equation}
which proves Eq.~\eqref{eq:leap_error_bound}.

For the bandwise result, define
\begin{equation}
    f_B(x,\rho)=\mathcal{P}_B f(x,\rho),
    \qquad
    R_B(x,\rho,h)=\mathcal{P}_B R(x,\rho,h).
\end{equation}
Applying the same argument to $f_B$ and $R_B$ gives
\begin{equation}
    \left\|
    \mathcal{P}_B
    \left(
    x^{\mathrm{micro}}_{\rho+H}
    -
    x^{\mathrm{leap}}_{\rho+H}
    \right)
    \right\|_F
    \le
    C H^2
    \left(
    L_{x,B}M_B + L_{\rho,B} + C_{R,B}
    \right),
\end{equation}
which proves the theorem.

\section{Pseudocode for StrideDiffusion}
\label{app:inference_pseudocode}

\paragraph{Notation.} We use $\rho_t=\bar{\alpha}_t$ as the cumulative signal coefficient,
with $\rho_{-1}\triangleq 1$ for boundary handling. The model $M_\theta$ is the $\hat{x}_0$-predictor of Diffusion-TS, from which we obtain
$\hat{\epsilon}_\theta(x_t,t)=(x_t-\sqrt{\rho_t}\hat{x}_0(x_t,t))/\sqrt{1-\rho_t}$.
The frequency partition $\mathcal{B}=\{B_0,\ldots,B_{|\mathcal{B}|-1}\}$ is built over the real-FFT bins of length $L$ ($F=\lfloor L/2\rfloor+1$), with the low-frequency subset denoted
$\mathcal{B}_{\mathrm{low}}\subseteq\mathcal{B}$. We denote the active set returned by the gate as $\mathcal{A}_t$, the leap schedule by
$(l_{\mathrm{coarse}},l_{\mathrm{mid}},l_{\mathrm{fine}})$,
the late-stage micro-step horizon by $K_{\mathrm{micro}}$, the dual-gate thresholds by $(\tau_{\mathrm{energy}},\tau_{\mathrm{mag}},\tau_{\mathrm{phase}})$, the DDIM stochasticity by $\eta$, and the optional soft-projection decay by $\gamma_d\in[0,1]$.

\begin{algorithm}[!h]
\caption{StrideDiffusion: spectral-guided fast inference}
\label{alg:stride_diffusion}
\begin{algorithmic}[1]
\Require Sample shape $(B,L,d)$; trained $\hat{x}_0$-predictor $M_\theta$ with $T$ training steps and schedule $\{\rho_t\}_{t=0}^{T-1}$;
band partition $\mathcal{B}$ with index sets $\{\mathcal{I}_B\}$ and low-frequency set $\mathcal{B}_{\mathrm{low}}$; leap schedule
$(l_{\mathrm{coarse}},l_{\mathrm{mid}},l_{\mathrm{fine}})$; late micro horizon $K_{\mathrm{micro}}$; thresholds $(\tau_{\mathrm{energy}},\tau_{\mathrm{mag}},\tau_{\mathrm{phase}})$; DDIM noise level $\eta$.
\Ensure Generated clean sample $x_0$.
\State Sample $x \sim \mathcal{N}(0,I)$ with shape $(B,L,d)$
\State $x_{\mathrm{prev}} \gets x$;\quad $\hat{\epsilon}_{\mathrm{prev}} \gets \mathrm{None}$;\quad $t \gets T-1$
\While{$t \ge 0$}
  \State \textbf{(1) Adaptive phase threshold.}
        $\mathrm{prog}\gets 1-t/T$;\quad
        $\tau_{\mathrm{phase}}^{(t)}\gets \tau_{\mathrm{phase}}\cdot(1-0.5\,\mathrm{prog})$
        \Comment{tighter phase gate late in sampling}
  \State \textbf{(2) Band-activity gate.}
        $\mathcal{A}_t \gets$ \Call{BandActivity}{$x_{\mathrm{prev}},x,\mathcal{B},
        \tau_{\mathrm{energy}},\tau_{\mathrm{mag}},\tau_{\mathrm{phase}}^{(t)}$}
        \Comment{Alg.~\ref{alg:band_activity}}
  \State \textbf{(3) Adaptive jump size.}
  \If{$t \le K_{\mathrm{micro}}$}
        \State $k \gets l_{\mathrm{fine}}$
        \Comment{always micro-step in the late, low-noise regime}
  \ElsIf{$\mathcal{A}_t = \emptyset$}
        \State $k \gets \min(l_{\mathrm{coarse}},\,t)$
  \ElsIf{$\mathcal{A}_t \subseteq \mathcal{B}_{\mathrm{low}}$}
        \State $k \gets \min(l_{\mathrm{mid}},\,t)$
  \Else
        \State $k \gets \min(l_{\mathrm{fine}},\,t)$
  \EndIf
  \State $t_{\mathrm{next}} \gets t-k$
  \State \textbf{(4) Single forward pass.}
        $\hat{x}_0 \gets \mathrm{clip}\!\bigl(M_\theta(x,t),\,-1,\,1\bigr)$
  \State \hphantom{\textbf{(4) Single forward pass.}}\,
        $\hat{\epsilon}_{\mathrm{curr}} \gets
         \bigl(x-\sqrt{\rho_t}\,\hat{x}_0\bigr)\big/\sqrt{1-\rho_t}$
  \State \textbf{(5) Deterministic single-step update $\Psi_{t,t_{\mathrm{next}}}$.}
  \If{$\hat{\epsilon}_{\mathrm{prev}} \neq \mathrm{None}$ \textbf{and} $k > 1$
       \textbf{and} $t_{\mathrm{next}} \ge 0$}
        \State $x_{\mathrm{cand}} \gets$
               \Call{DPMSolver2Jump}{$x,t,t_{\mathrm{next}},
               \hat{\epsilon}_{\mathrm{curr}},\hat{\epsilon}_{\mathrm{prev}},\rho$}
               \Comment{Alg.~\ref{alg:dpm2_jump}}
  \Else
        \State $x_{\mathrm{cand}} \gets$
               \Call{DDIMJump}{$x,t,t_{\mathrm{next}},
               \hat{\epsilon}_{\mathrm{curr}},\hat{x}_0,\rho,\eta$}
               \Comment{Alg.~\ref{alg:ddim_jump}}
  \EndIf
  \State $\hat{\epsilon}_{\mathrm{prev}} \gets \hat{\epsilon}_{\mathrm{curr}}$
  \State \textbf{(6) (Optional) Soft band projection.}
  \If{\textsc{useProjection} \textbf{and} $\mathcal{A}_t \neq \emptyset$}
        \State $\Delta \gets x_{\mathrm{cand}}-x$;\quad
              $\mathcal{A}_t^{\mathrm{proj}} \gets \mathcal{A}_t \cup \mathcal{B}_{\mathrm{low}}$
        \State $x_{\mathrm{new}} \gets x +
              \Call{BandSoftProject}{\Delta,\mathcal{B},\mathcal{A}_t^{\mathrm{proj}},
              \gamma_d}$
              \Comment{keeps active bands, attenuates others by $\gamma_d$}
  \Else
        \State $x_{\mathrm{new}} \gets x_{\mathrm{cand}}$
  \EndIf
  \State \textbf{(7) Advance.}
        $x_{\mathrm{prev}} \gets x$;\quad $x \gets x_{\mathrm{new}}$;\quad $t \gets t_{\mathrm{next}}$
\EndWhile
\State \Return $x$
\end{algorithmic}
\end{algorithm}

\begin{algorithm}[!h]
\caption{\textsc{BandActivity}: dual energy / dynamics gate over rFFT bands}
\label{alg:band_activity}
\begin{algorithmic}[1]
\Require Two consecutive states $x_{t-1},x_t \in \mathbb{R}^{B\times L\times d}$;
band partition $\mathcal{B}$ with index sets $\{\mathcal{I}_B\}$;
thresholds $(\tau_{\mathrm{energy}},\tau_{\mathrm{mag}},\tau_{\mathrm{phase}})$;
small constant $\varepsilon>0$.
\Ensure Active band set $\mathcal{A}_t \subseteq \mathcal{B}$.
\State $X_{t-1}\gets\mathrm{rFFT}(x_{t-1},\text{dim}=L)$;\quad
       $X_t\gets\mathrm{rFFT}(x_t,\text{dim}=L)$
       \Comment{both in $\mathbb{C}^{B\times F\times d}$}
\State $p_{t-1}[f]\gets\sum_{j=1}^{d}|X_{t-1}[f,j]|^2$;\quad
       $p_t[f]\gets\sum_{j=1}^{d}|X_t[f,j]|^2$
\For{each band $B \in \mathcal{B}$}
   \State $P_{t-1}[B]\gets|\mathcal{I}_B|^{-1}\sum_{f\in\mathcal{I}_B} p_{t-1}[f]$;\quad
          $P_t[B]\gets|\mathcal{I}_B|^{-1}\sum_{f\in\mathcal{I}_B} p_t[f]$
   \State $\pi_t[B]\gets P_t[B]\big/\sum_{B'\in\mathcal{B}} P_t[B']$
          \Comment{relative band energy}
   \State $\delta_t[B]\gets
          \bigl|\log(P_t[B]+\varepsilon)-\log(P_{t-1}[B]+\varepsilon)\bigr|$
          \Comment{log-power drift}
   \State $\phi_t[f,j]\gets\angle\!\bigl(X_t[f,j]\,\overline{X_{t-1}[f,j]}\bigr)$
          for $f\in\mathcal{I}_B,\;j=1,\ldots,d$
          \Comment{per-channel phase difference}
   \State $\bar{\phi}_t[f]\gets d^{-1}\sum_{j=1}^{d}\phi_t[f,j]$
          \Comment{average across channels first}
   \State $w_t[f]\gets p_t[f]+\varepsilon$
   \State $v_{\phi,t}[B]\gets
          \sqrt{\sum_{f\in\mathcal{I}_B} w_t[f]\,\bar{\phi}_t[f]^2 \big/
                \sum_{f\in\mathcal{I}_B} w_t[f]}$
          \Comment{power-weighted phase velocity}
\EndFor
\State Reduce $(\pi_t,\delta_t,v_{\phi,t})$ over the batch by mean
\State $\mathcal{A}_t \gets \emptyset$
\For{$i=0,\ldots,|\mathcal{B}|-1$}
   \State $s_i \gets s_{\mathrm{hi}}$ \textbf{if} $i>0$ \textbf{else} $1$
          \Comment{phase boost $s_{\mathrm{hi}}\!\ge\!1$ on non-DC bands}
   \If{$\pi_t[B_i]\ge\tau_{\mathrm{energy}}$
        \textbf{and}
        $\bigl(\delta_t[B_i]\ge\tau_{\mathrm{mag}}
        \;\textbf{or}\; s_i\,v_{\phi,t}[B_i]\ge\tau_{\mathrm{phase}}\bigr)$}
        \State $\mathcal{A}_t \gets \mathcal{A}_t \cup \{B_i\}$
   \EndIf
\EndFor
\State \Return $\mathcal{A}_t$
\end{algorithmic}
\end{algorithm}

\begin{algorithm}[!h]
\caption{\textsc{DDIMJump}: deterministic DDIM update from $t$ to $t_{\mathrm{next}}$}
\label{alg:ddim_jump}
\begin{algorithmic}[1]
\Require Current state $x_t$; indices $t,t_{\mathrm{next}}$;
predicted noise $\hat{\epsilon}_\theta(x_t,t)$ and clean sample $\hat{x}_0(x_t,t)$;
schedule $\{\rho_t\}$; stochasticity $\eta\ge 0$.
\Ensure $x_{t_{\mathrm{next}}}$.
\State Use $\rho_t$ and $\rho_{t_{\mathrm{next}}}$ from the schedule, with the boundary
       convention $\rho_{-1}\triangleq 1$
\State $\sigma \gets
        \eta\sqrt{\max\!\bigl((1-\rho_{t_{\mathrm{next}}}/\rho_t)(1-\rho_t)/(1-\rho_{t_{\mathrm{next}}}+\varepsilon),0\bigr)}$
\State $c \gets \sqrt{\max\!\bigl(1-\rho_{t_{\mathrm{next}}}-\sigma^2,\;0\bigr)}$
\State $z \gets \mathcal{N}(0,I)$ \textbf{if} $\sigma>0$ \textbf{else} $0$
\State \Return $\sqrt{\rho_{t_{\mathrm{next}}}}\,\hat{x}_0(x_t,t) + c\,\hat{\epsilon}_\theta(x_t,t) + \sigma z$
\end{algorithmic}
\end{algorithm}

\begin{algorithm}[!h]
\caption{\textsc{DPMSolver2Jump}: DPM-Solver-2 multistep jump (uses $\hat{\epsilon}_{\mathrm{prev}}$)}
\label{alg:dpm2_jump}
\begin{algorithmic}[1]
\Require Current state $x_t$; indices $t,t_{\mathrm{next}}$;
current and previous noise predictions $\hat{\epsilon}_{\mathrm{curr}},\hat{\epsilon}_{\mathrm{prev}}$;
schedule $\{\rho_t\}$; history-step constant $h_{\mathrm{prev}}>0$.
\Ensure $x_{t_{\mathrm{next}}}$.
\State $\lambda(\rho)\triangleq\tfrac{1}{2}\log\!\bigl(\rho/(1-\rho)\bigr)$
       \Comment{half log-SNR}
\State $\lambda_t\gets\lambda(\rho_t)$;\quad
       $\lambda_{\mathrm{next}}\gets\lambda(\rho_{t_{\mathrm{next}}})$
\State $h \gets \lambda_{\mathrm{next}}-\lambda_t$;\quad $r \gets h_{\mathrm{prev}}/h$
\State $\hat{\epsilon}^{(2)} \gets
        \hat{\epsilon}_{\mathrm{curr}} + \tfrac{1}{2r}\bigl(\hat{\epsilon}_{\mathrm{curr}}
        -\hat{\epsilon}_{\mathrm{prev}}\bigr)$
       \Comment{second-order extrapolation}
\State \Return $\sqrt{\rho_{t_{\mathrm{next}}}/\rho_t}\,x_t
        - \sqrt{1-\rho_{t_{\mathrm{next}}}}\,\bigl(e^{h}-1\bigr)\,\hat{\epsilon}^{(2)}$
\end{algorithmic}
\end{algorithm}

\paragraph{Optional rFFT soft projection.}
When \textsc{useProjection} is enabled, the per-step increment $\Delta=x_{\mathrm{cand}}-x$ is projected onto the rFFT bins of the active bands. Letting $\mathcal{F}$ denote the rFFT along the time axis and $M$ a diagonal mask that equals $1$ on $\bigcup_{B\in\mathcal{A}_t^{\mathrm{proj}}}\mathcal{I}_B$ and $\gamma_d\in[0,1]$ elsewhere,
we set
\begin{equation}
    \textsc{BandSoftProject}(\Delta,\mathcal{B},\mathcal{A}_t^{\mathrm{proj}},\gamma_d)
    =
    \mathcal{F}^{-1}\bigl(M\odot \mathcal{F}\Delta\bigr).
\end{equation}
Setting $\gamma_d=0$ recovers the hard projector $\mathcal{P}_{\mathcal{A}_t^{\mathrm{proj}}}$ analyzed in Sec.~\ref{sec:theory_spectral_stability}; we use a small $\gamma_d>0$ to avoid abrupt cliff edges in the spectrum across consecutive jumps.

\paragraph{Conditional inference.}
For imputation and forecasting (Tables~\ref{tab:cond_impute}-\ref{tab:cond_forecast}), Algorithm~\ref{alg:stride_diffusion} is reused with two modifications inside the while loop, mirroring the Diffusion-TS conditional sampler. (i) After step (5) we optionally apply a Langevin refinement on $x_{\mathrm{new}}$ that minimizes $\|\hat{x}_0[\,\mathrm{mask}\,]-\mathrm{target}[\,\mathrm{mask}\,]\|^2$ under the model's posterior at level $t$. (ii) After step (6) we overwrite the observed entries with $q(\cdot\mid\mathrm{target},t_{\mathrm{next}})$ so that the known region remains consistent with the forward noise schedule at the post-jump level. Implementation details follow \texttt{sample\_infill\_banded} in our code release.

\section{Speed-Quality Pareto Frontier on Unconditional Generation}
\label{sec:pareto_appendix}

\begin{figure}[!h]
\centering
\includegraphics[width=\linewidth]{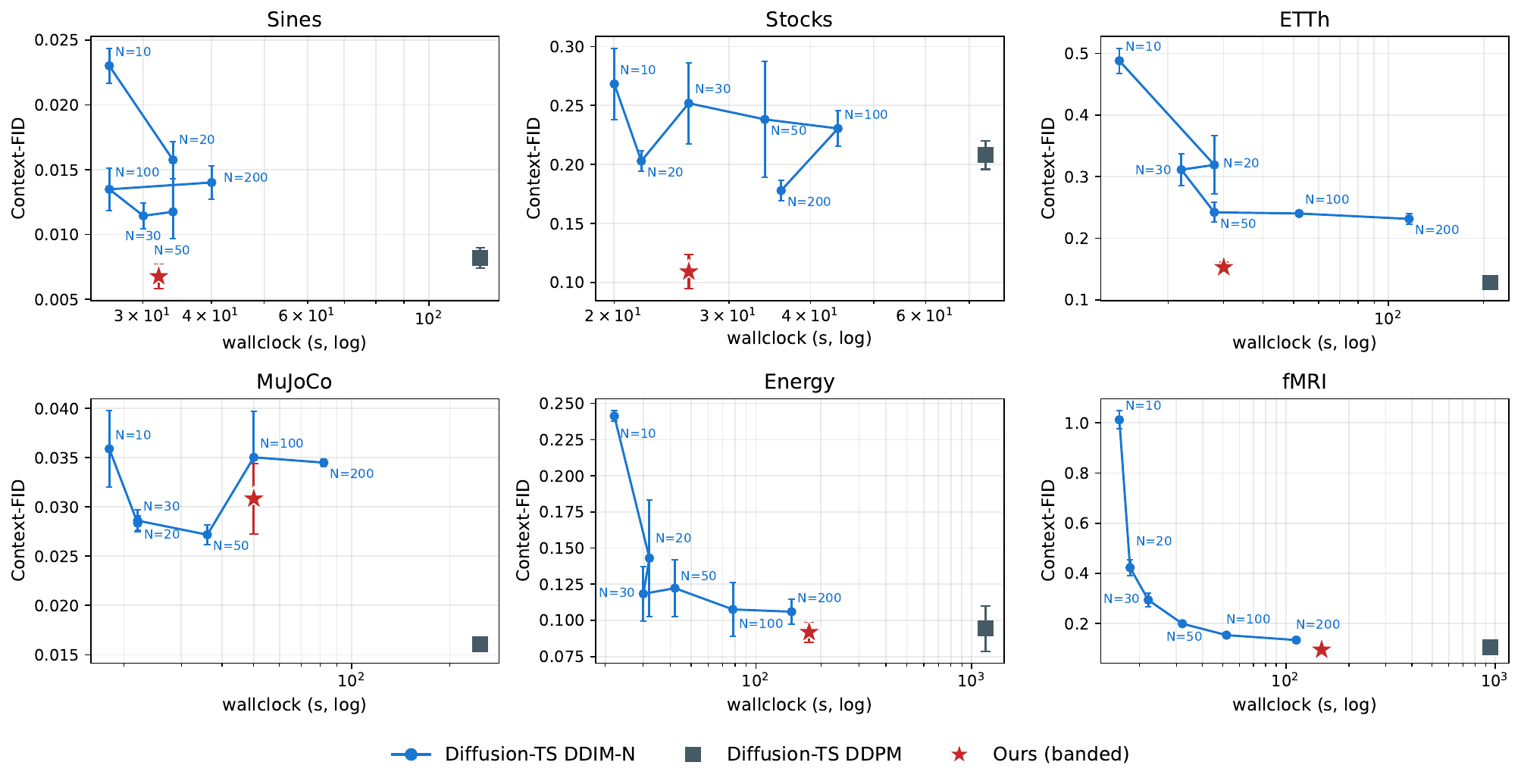}
\caption{Speed-quality Pareto frontier on unconditional generation. Each panel plots wallclock (log-scale, seconds) vs.\ Context-FID for one dataset. Blue: Diffusion-TS DDIM-$N$ for $N\in\{10,20,30,50,100,200\}$ (smaller $N$ = faster, larger $N$ = closer to full DDPM). Gray square: full-$T$ DDPM. Red star: our band-aware sampler with the per-dataset schedule of Sec.~\ref{sec:exp_setup}. Error bars are $\pm 1$ standard deviation across three Context-FID re-evaluations on a fixed sample set.}
\label{fig:pareto_uncond}
\end{figure}

To complement Table~\ref{tab:results-unconditional-generation}, we contrast our sampler with a wider set of DDIM step budgets on the same Diffusion-TS backbone. For each dataset, we run vanilla full-$T$ DDPM and DDIM-$N$ for $N\in\{10,20,30,50,100,200\}$, recording wallclock and Context-FID per (method, $N$) cell. Figure~\ref{fig:pareto_uncond} plots the resulting (wallclock, Context-FID) scatter on log-scaled axes, with the DDIM-$N$ points connected to expose the speed-quality trade-off as $N$ varies.

On \textsc{sines}, \textsc{stocks}, \textsc{etth}, \textsc{energy}, and \textsc{fmri}, our sampler (red star) lies strictly below the DDIM-$N$ curve, i.e. it dominates every uniform-stride budget in \emph{both} quality and wallclock. \textsc{mujoco} is the one exception: full-$T$ DDPM drives Context-FID down to $0.016$ that no DDIM budget recovers, and our sampler matches the best DDIM-$N$ point ($\sim 0.030$) at slightly higher cost. We attribute this to \textsc{mujoco}'s simulated, smoothly-mixing dynamics, which appear to benefit disproportionately from full-$T$ stochastic ancestral sampling rather than from any deterministic accelerator. The remaining five datasets, especially the real-world ones (\textsc{stocks}, \textsc{etth}, \textsc{energy}, \textsc{fmri}), show the band-aware schedule extracts more quality per second than uniform DDIM stride at every budget we tried.

\section{Conditional Generation under StrideDiffusion}

\begin{table}[h!]
  \centering
  \caption{Conditional generation results for Imputation across missing ratios. Mean across seeds.}
  \label{tab:cond_impute_lv}
  \resizebox{\linewidth}{!}{%
  \renewcommand\arraystretch{0.99}
\begin{tabular}{clcccccccc}
    \toprule
    Missing Ratio & Method & \multicolumn{2}{c}{Stocks} & \multicolumn{2}{c}{ETTh} & \multicolumn{2}{c}{Energy} & \multicolumn{2}{c}{fMRI} \\
    \cmidrule(lr){3-4} \cmidrule(lr){5-6} \cmidrule(lr){7-8} \cmidrule(lr){9-10}
     &  & MSE $\downarrow$ & Time(s) $\downarrow$ & MSE $\downarrow$ & Time(s) $\downarrow$ & MSE $\downarrow$ & Time(s) $\downarrow$ & MSE $\downarrow$ & Time(s) $\downarrow$ \\
    \midrule
    \multirow{4}{*}{0.1} & Diffusion-TS & 0.010 & 56.97 & 0.002 & 214.69 & 0.011 & 716.59 & 0.017 & 436.07 \\
     & Diffusion-TS-fast & 0.009 & 22.75 & 0.002 & 88.92 & 0.010 & 146.66 & 0.026 & 85.99 \\
     & Ours w/o Langevin & 0.008 & 8.79 & 0.003 & 12.28 & 0.011 & 156.95 & 0.022 & 22.90 \\
     & Ours with Langevin & \textbf{0.008} & 25.44 & \textbf{0.002} & 18.91 & \textbf{0.010} & 120.78 & \textbf{0.015} & 96.22 \\
    \midrule
    \multirow{4}{*}{0.25} & Diffusion-TS & 0.010 & 55.93 & 0.002 & 210.79 & 0.013 & 704.15 & 0.018 & 418.51 \\
     & Diffusion-TS-fast & 0.009 & 22.37 & 0.003 & 89.02 & 0.012 & 139.57 & 0.028 & 86.54 \\
     & Ours w/o Langevin & 0.008 & 9.13 & 0.004 & 12.20 & 0.014 & 175.00 & 0.023 & 26.78 \\
     & Ours with Langevin & \textbf{0.008} & 15.47 & 0.003 & 11.75 & \textbf{0.013} & 119.17 & \textbf{0.015} & 92.91 \\
    \midrule
    \multirow{4}{*}{0.5} & Diffusion-TS & 0.011 & 56.24 & 0.003 & 217.75 & 0.016 & 699.99 & 0.020 & 420.23 \\
     & Diffusion-TS-fast & 0.009 & 22.98 & 0.003 & 89.67 & 0.016 & 135.56 & 0.033 & 82.77 \\
     & Ours w/o Langevin & 0.008 & 11.53 & 0.007 & 10.98 & 0.019 & 177.47 & 0.025 & 18.90 \\
     & Ours with Langevin & \textbf{0.008} & 8.33 & \textbf{0.003} & 19.06 & \textbf{0.016} & 121.43 & \textbf{0.017} & 86.08 \\
    \midrule
    \multirow{4}{*}{0.75} & Diffusion-TS & 0.012 & 56.46 & 0.004 & 219.53 & 0.019 & 681.06 & 0.022 & 420.89 \\
     & Diffusion-TS-fast & 0.009 & 23.71 & 0.004 & 86.64 & 0.019 & 140.39 & 0.037 & 86.03 \\
     & Ours w/o Langevin& 0.018 & 16.74 & 0.015 & 8.99 & 0.032 & 179.90 & 0.028 & 24.17 \\
     & Ours with Langevin & \textbf{0.008} & 7.30 & 0.005 & 18.99 & \textbf{0.019} & 120.73 & \textbf{0.021} & 78.07 \\
    \midrule
    \multirow{4}{*}{0.9} & Diffusion-TS & 0.012 & 55.44 & 0.004 & 213.46 & 0.019 & 676.07 & 0.022 & 418.91 \\
     & Diffusion-TS-fast & 0.009 & 22.04 & 0.004 & 98.13 & 0.019 & 137.64 & 0.037 & 87.28 \\
     & Ours w/o Langevin & 0.018 & 12.82 & 0.015 & 8.33 & 0.033 & 171.46 & 0.028 & 19.09 \\
     & Ours with Langevin & \textbf{0.008} & 7.61 & 0.005 & 17.65 & \textbf{0.019} & 129.11 & \textbf{0.021} & 77.04 \\
 \midrule
\multicolumn{2}{c}{\textbf{Avg. Speedup vs Diffusion-TS}}    
& \multicolumn{2}{c}{\cellcolor{blue!10} \textbf{5.53$\times$}}                   
& \multicolumn{2}{c}{\cellcolor{blue!10} \textbf{12.87$\times$}}             
& \multicolumn{2}{c}{\cellcolor{blue!10} \textbf{5.70$\times$}}             
& \multicolumn{2}{c}{\cellcolor{blue!10} \textbf{4.95$\times$}}             \\
\bottomrule   
  \end{tabular}}
\end{table}

\begin{table}[h!]
  \centering
  \caption{Conditional generation results for Forecasting across prediction lengths. Mean across seeds.}
  \label{tab:cond_forecast}
  \resizebox{\linewidth}{!}{%
  \renewcommand\arraystretch{0.99}
  \begin{tabular}{clcccccccc}
    \toprule
    \multirow{2}{*}{Length} & \multirow{2}{*}{Method} & \multicolumn{2}{c}{Stocks} & \multicolumn{2}{c}{ETTh} & \multicolumn{2}{c}{Energy} & \multicolumn{2}{c}{fMRI} \\
    \cmidrule(lr){3-4} \cmidrule(lr){5-6} \cmidrule(lr){7-8} \cmidrule(lr){9-10}
     &  & MSE $\downarrow$ & Time(s) $\downarrow$ & MSE $\downarrow$ & Time(s) $\downarrow$ & MSE $\downarrow$ & Time(s) $\downarrow$ & MSE $\downarrow$ & Time(s) $\downarrow$ \\
    \midrule
    \multirow{2}{*}{6} & Diffusion-TS & 0.027 & 82.00 & 0.010 & 291.75 & 0.026 & 965.94 & 0.024 & 662.12 \\
     & Ours & \textbf{0.010} & \textbf{18.41} & 0.011 & 19.66 & \textbf{0.025} & \textbf{163.74} & \textbf{0.022} & \textbf{108.43} \\
    \midrule
    \multirow{2}{*}{12} & Diffusion-TS & 0.032 & 89.58 & 0.011 & 339.02 & 0.026 & 1043.88 & 0.025 & 631.89 \\
     & Ours& \textbf{0.009} & \textbf{17.54} & 0.013 & 23.84 & \textbf{0.025} & \textbf{161.75} & \textbf{0.023} & \textbf{116.28} \\
    \midrule
    \multirow{2}{*}{24} & Diffusion-TS & 0.028 & 85.30 & 0.011 & 270.15 & 0.027 & 1021.27 & 0.027 & 573.39 \\
     & Ours & \textbf{0.008} & \textbf{18.49} & 0.014 & 20.89 & \textbf{0.026} & \textbf{165.20} & \textbf{0.024} & \textbf{117.82} \\
    \midrule
    \multirow{2}{*}{36} & Diffusion-TS & 0.021 & 110.82 & 0.013 & 300.66 & 0.029 & 977.73 & 0.027 & 635.09 \\
     & Ours   & \textbf{0.009} & \textbf{20.32} & 0.018 & 21.88 & \textbf{0.028} & \textbf{157.23} & \textbf{0.025} & \textbf{113.77} \\
 \midrule
\multicolumn{2}{c}{\textbf{Avg. Speedup vs Diffusion-TS}}    
& \multicolumn{2}{c}{\cellcolor{blue!10} \textbf{4.91$\times$}}                   
& \multicolumn{2}{c}{\cellcolor{blue!10} \textbf{13.93$\times$}}             
& \multicolumn{2}{c}{\cellcolor{blue!10} \textbf{6.19$\times$}}             
& \multicolumn{2}{c}{\cellcolor{blue!10} \textbf{5.50$\times$}}             \\
\bottomrule
\end{tabular}}
\end{table}

Our conditional sampler keeps the optional Langevin step from Diffusion-TS, which pulls the observed entries back toward their target after each denoising step. Tables~\ref{tab:cond_impute_lv} and~\ref{tab:cond_forecast} report both versions. At low missing ratios and short horizons the two are nearly identical: Stocks imputation sits at $0.008$ MSE up to ratio $0.5$ either way. The picture changes once the task gets harder. ETTh imputation jumps from $0.005$ to $0.015$ at ratio $0.75$ when we drop Langevin, and Stocks forecasting at horizon $36$ goes from $0.015$ to $0.072$.

\section{Ablation: Components of Frequency-Aware Sampling}
\label{sec:ablation_appendix}

The frequency-aware sampler exposes a small set of hyperparameters
($l_{\text{coarse}}$, $l_{\text{mid}}$, $l_{\text{fine}}$, the late-step micro window $K_{\mathrm{micro}}$, and the gating thresholds $\tau_{\mathrm{energy}}$, $\tau_{\mathrm{mag}}$,
$\tau_{\mathrm{phase}}$) that trade off sampling speed against generation quality.

Both axes vary across datasets: highly periodic signals such as \textsc{sines} tolerate large jumps and admit larger speedups, while spectrally rich signals such as \textsc{fmri} reward more conservative
schedules with tighter quality tracking. We therefore do not seek a single universal configuration. Instead, for each dataset we select the hyperparameter setting that balances aggressive speedup against Context-FID parity with vanilla full-$T$ DDPM, and report this configuration as our \textbf{Full} model. All ablation variants share the same per-dataset checkpoint and evaluation protocol; only the ablated ingredient is changed (Tables~\ref{tab:ablation_A} and~\ref{tab:ablation_A_full}).

We compare five settings:
\textbf{Full (ours)} is the complete frequency-aware sampler tuned per dataset for the speed-quality balance described above;
\textbf{Vanilla DDPM} disables the adaptive leap schedule entirely (equivalent to a full-$T$ ancestral sampler);
\textbf{w/o gate} removes the band-activity gate, forcing the maximum jump $l_{\text{coarse}}$ at every step except the late-step refinement window;
\textbf{w/o late-step micro} sets the late-step refinement window to zero; and \textbf{w/o }$\tau_{\mathrm{energy}}$ disables the energy threshold inside the gate so that band activity is decided purely by phase and magnitude signals.

\paragraph{Large speedup with quality on par with vanilla DDPM.}
Across all four datasets, our balanced configuration delivers a $13$-$21\times$ wallclock speedup over vanilla full-$T$ DDPM while keeping Context-FID within a tight band of the vanilla baseline (better on \textsc{sines}, \textsc{stocks}, \textsc{fmri}; within $8\%$ on \textsc{energy}). The exact speed-quality trade-off is
dataset-dependent \textsc{energy} admits the largest acceleration ($21\times$) while \textsc{stocks} sits at $13\times$ but the qualitative conclusion is uniform: the standard single-step DDPM trajectory is significantly over-resolved for time-series generation, and a band-aware schedule recovers an order of magnitude in cost without measurable quality regression.

\paragraph{The band-activity gate is necessary for the speed gains to be safe.}
Removing the gate (\emph{w/o gate}) is the fastest variant by construction, but Context-FID degrades by $39$-$3680\%$ on \textsc{sines}, \textsc{stocks}, and \textsc{energy}. This shows that indiscriminate large jumps cannot be applied without the gate's band-aware decision; the speedups reported above hold only because the gate adapts the step size to the local frequency content. \textsc{fmri} is the only dataset on which the gate is unnecessary, which we attribute to its near-uniform spectral profile when band activity is roughly constant in time, the gating decision adds no information.

\paragraph{Energy is the load-bearing gate signal.}
Disabling the energy threshold (\emph{w/o }$\tau_{\mathrm{energy}}$) lowers the activity decision so that bands are flagged active too easily, forcing more small steps: sampling slows by up to $4\times$ relative to the full method while quality also degrades on every dataset ($+6$ to $+65\%$ Context-FID). This identifies energy, rather than the auxiliary phase or magnitude signals, as the primary driver of the gating behavior in practice.

\paragraph{Late-step micro refinement is a low-cost safeguard.}
Removing late-step micro refinement (\emph{w/o late-step micro}) hurts Context-FID on spectrally rich signals (\textsc{sines} $+14\%$, \textsc{fmri} $+11\%$) but is essentially a no-op on smoother sequences (\textsc{stocks} $-0.3\%$, \textsc{energy} $-8\%$). We retain it as part of the balanced default since its overhead is small ($<2\times$ additional steps) and it prevents quality regression on signals with sharp high-frequency structure.

\paragraph{Summary.}
The speedups offered by our sampler are large but vary across datasets, and so does the resulting generation quality. The balanced configurations reported as \textbf{Full} are chosen to be fast \emph{and} match vanilla full-$T$ DDPM in Context-FID. Ablation identifies the band-activity gate driven by the energy threshold as the indispensable component that makes the underlying jump schedule both fast and quality-preserving, with late-step refinement as a small, robust safeguard.

\begin{table}[!h]
\centering
\caption{Transposed full ablation of frequency-aware sampling components. Time is wallclock seconds for one full sampling pass; speedup is relative to vanilla full-$T$ DDPM; lower is better for C-FID and Discriminative score (Disc). \textbf{Bold} marks the best value per row, with vanilla excluded from time/speedup comparison. The \textbf{Full (ours)} configuration is the balanced choice that trades a small amount of speed for sample quality on par with vanilla DDPM.}
\label{tab:ablation_A_full}
\resizebox{\textwidth}{!}{%
\begin{tabular}{llrrrrr}
\toprule
Dataset & Metric & \textbf{Full (ours)} & Vanilla DDPM & w/o gate & w/o late-step micro & w/o $\tau_{\mathrm{energy}}$ \\
\midrule

\multirow{4}{*}{\textsc{sines}}
& Time (s) $\downarrow$ 
& 9.30 & 130.32 & \textbf{3.14} & 6.17 & 38.76 \\
& Speedup $\uparrow$ 
& 14.0$\times$ & 1.0$\times$ & \textbf{41.5$\times$} & 21.1$\times$ & 3.4$\times$ \\
& C-FID $\downarrow$ 
& \textbf{0.0086} & 0.0161 & 0.3263 & 0.0099 & 0.0103 \\
& Disc $\downarrow$ 
& 0.0129 & 0.0110 & 0.1182 & \textbf{0.0104} & 0.0365 \\

\midrule

\multirow{4}{*}{\textsc{stocks}}
& Time (s) $\downarrow$ 
& 5.94 & 78.53 & \textbf{2.63} & 5.90 & 26.58 \\
& Speedup $\uparrow$ 
& 13.2$\times$ & 1.0$\times$ & \textbf{29.9$\times$} & 13.3$\times$ & 3.0$\times$ \\
& C-FID $\downarrow$ 
& 0.1275 & 0.2398 & 0.4064 & \textbf{0.1270} & 0.1959 \\
& Disc $\downarrow$ 
& 0.0888 & 0.1112 & 0.1108 & \textbf{0.0760} & 0.0903 \\

\midrule

\multirow{4}{*}{\textsc{energy}}
& Time (s) $\downarrow$ 
& 56.46 & 1178.48 & \textbf{28.01} & 49.45 & 115.24 \\
& Speedup $\uparrow$ 
& 20.9$\times$ & 1.0$\times$ & \textbf{42.1$\times$} & 23.8$\times$ & 10.2$\times$ \\
& C-FID $\downarrow$ 
& 0.0883 & 0.0817 & 0.1231 & \textbf{0.0813} & 0.0942 \\
& Disc $\downarrow$ 
& 0.1233 & 0.1421 & 0.2629 & 0.1362 & \textbf{0.1226} \\

\midrule

\multirow{4}{*}{\textsc{fmri}}
& Time (s) $\downarrow$ 
& 44.41 & 638.29 & 47.28 & \textbf{39.16} & 68.58 \\
& Speedup $\uparrow$ 
& 14.4$\times$ & 1.0$\times$ & 13.5$\times$ & \textbf{16.3$\times$} & 9.3$\times$ \\
& C-FID $\downarrow$ 
& \textbf{0.0937} & 0.1154 & 0.0937 & 0.1039 & 0.1548 \\
& Disc $\downarrow$ 
& 0.1021 & 0.1632 & \textbf{0.0981} & 0.1636 & 0.1874 \\

\bottomrule
\end{tabular}}
\end{table}

\section{Hyperparameter Sensitivity on All Datasets}
\label{sec:sensitivity_appendix}

We extend the sensitivity analysis of Sec.~\ref{sec:sensitivity} to all six datasets used in the unconditional generation experiments. For each dataset we sweep $l_{\mathrm{coarse}}\!\in\!\{10, 20, 30, 50, 100\}$, the late-step micro window $K_{\mathrm{micro}}\!\in\!\{0, 4, 8, 12, 20\}$, and $\tau_{\mathrm{phase}}\!\in\!\{0.02, 0.04, 0.08, 0.16, 0.32\}$ around its per-dataset balanced default (dotted vertical line in each panel), holding the remaining hyperparameters fixed. The qualitative patterns are consistent across datasets: $l_{\mathrm{coarse}}$ exhibits a sweet spot at the default, $K_{\mathrm{micro}}$ has a threshold-like effect, and $\tau_{\mathrm{phase}}$ has essentially no measurable effect on either C-FID or sampling time within the swept range. Together these confirm that the per-dataset hyperparameter choices used in our main experiments are well calibrated and that the sampler is robust to the choice of $\tau_{\mathrm{phase}}$.

\begin{figure}[!h]
\centering
\subfloat[\textsc{sines}\label{fig:hyper_appx_sines}]{%
\includegraphics[width=\linewidth]{./figs/hyper_sines.pdf}}\\[0.3em]
\subfloat[\textsc{stocks}\label{fig:hyper_appx_stocks}]{%
\includegraphics[width=\linewidth]{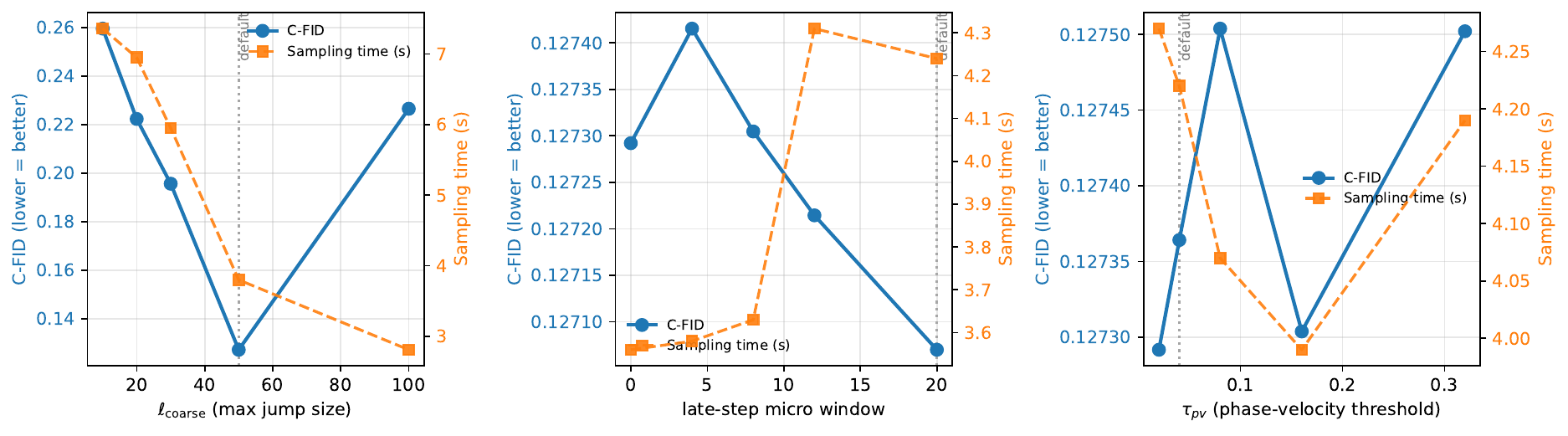}}\\[0.3em]
\subfloat[\textsc{etth}\label{fig:hyper_appx_etth}]{%
\includegraphics[width=\linewidth]{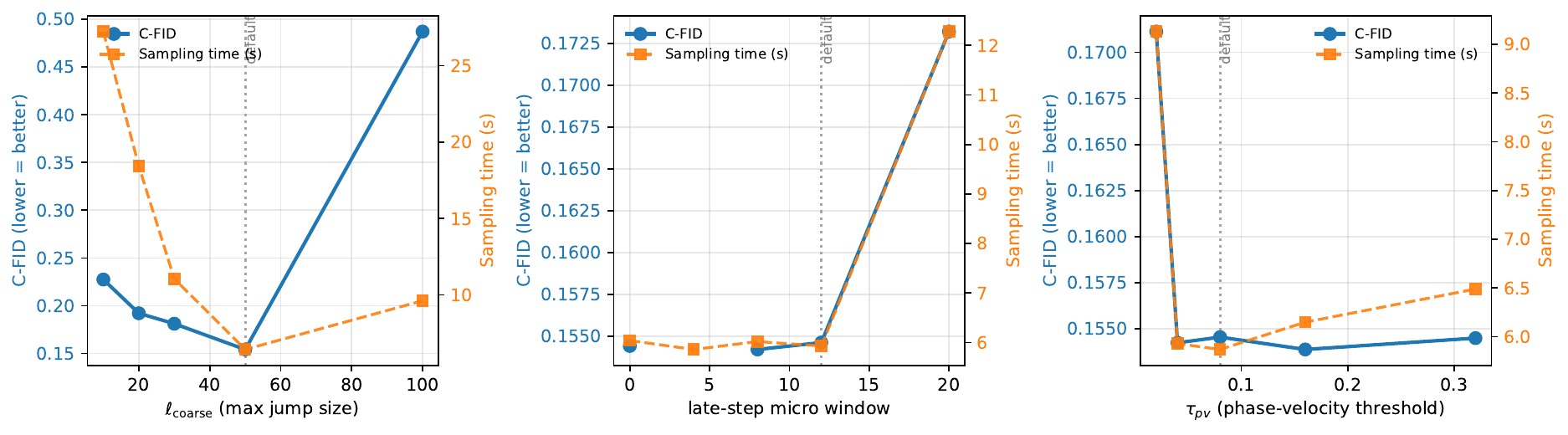}}
\caption{Hyperparameter sensitivity on (a) \textsc{sines}, (b) \textsc{stocks}, and (c) \textsc{etth}. Same protocol as Fig.~\ref{fig:hyper}: in each row, $l_{\mathrm{coarse}}$ shows a U-shape with the minimum at the default (dotted line), $K_{\mathrm{micro}}$ has a threshold-like effect, and $\tau_{\mathrm{phase}}$ has no measurable effect.}

\label{fig:hyper_appx_1}
\end{figure}

\begin{figure}[!h]
\centering
\subfloat[\textsc{mujoco}\label{fig:hyper_appx_mujoco}]{%
\includegraphics[width=\linewidth]{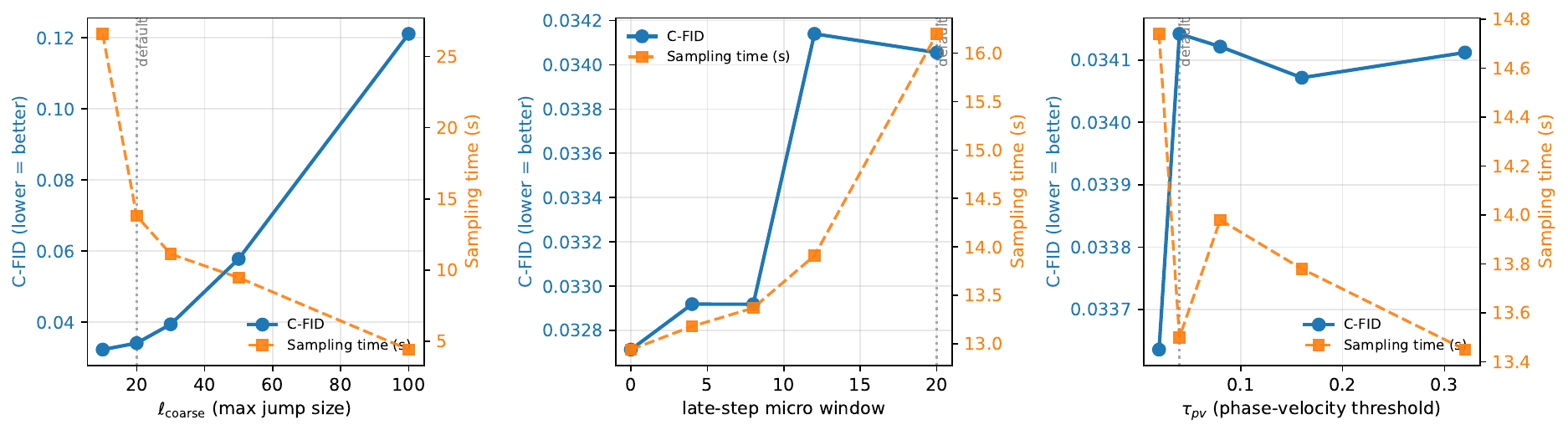}}\\[0.3em]
\subfloat[\textsc{energy}\label{fig:hyper_appx_energy}]{%
\includegraphics[width=\linewidth]{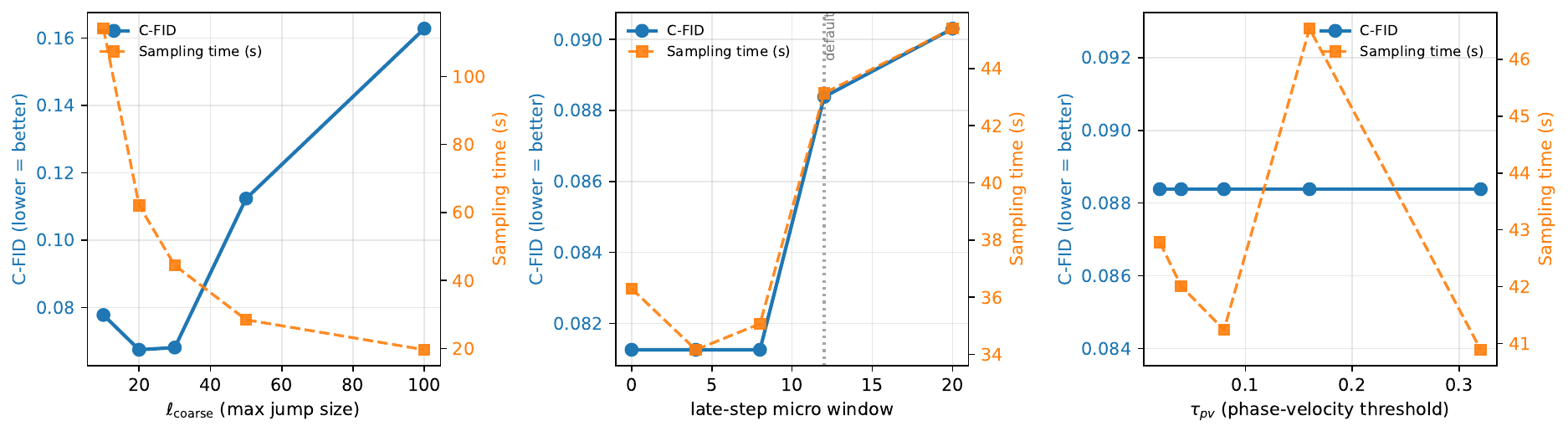}}\\[0.3em]
\subfloat[\textsc{fmri}\label{fig:hyper_appx_fmri}]{%
\includegraphics[width=\linewidth]{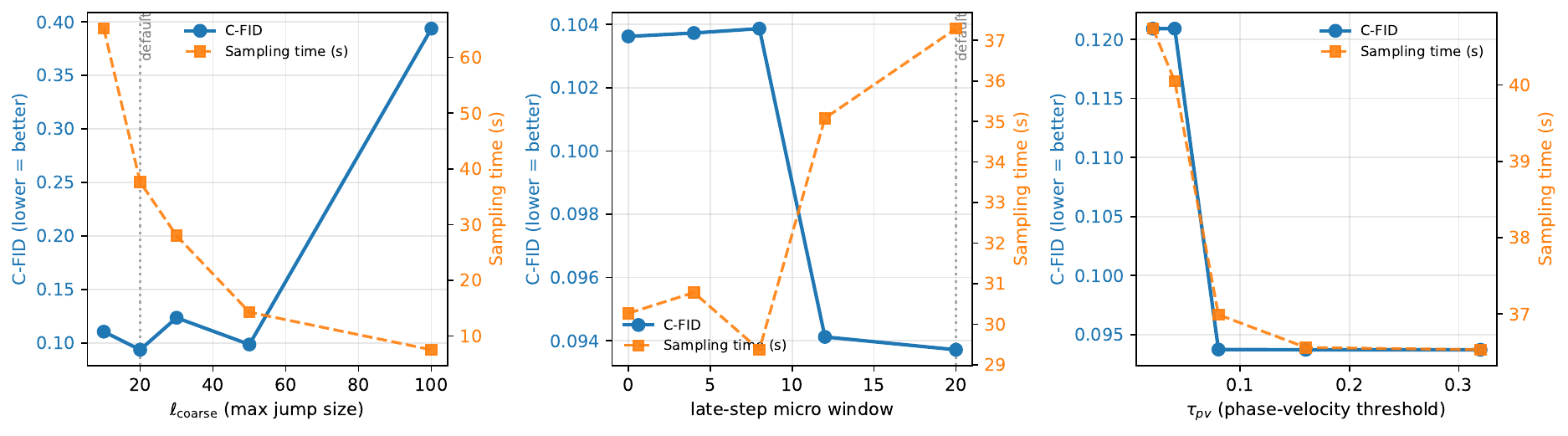}}
\caption{Hyperparameter sensitivity on (a) \textsc{mujoco}, (b) \textsc{energy}, and (c) \textsc{fmri}. Same protocol as Fig.~\ref{fig:hyper}.}
\label{fig:hyper_appx_2}
\end{figure}

\begin{figure}[!h]
\centering
\includegraphics[width=\linewidth]{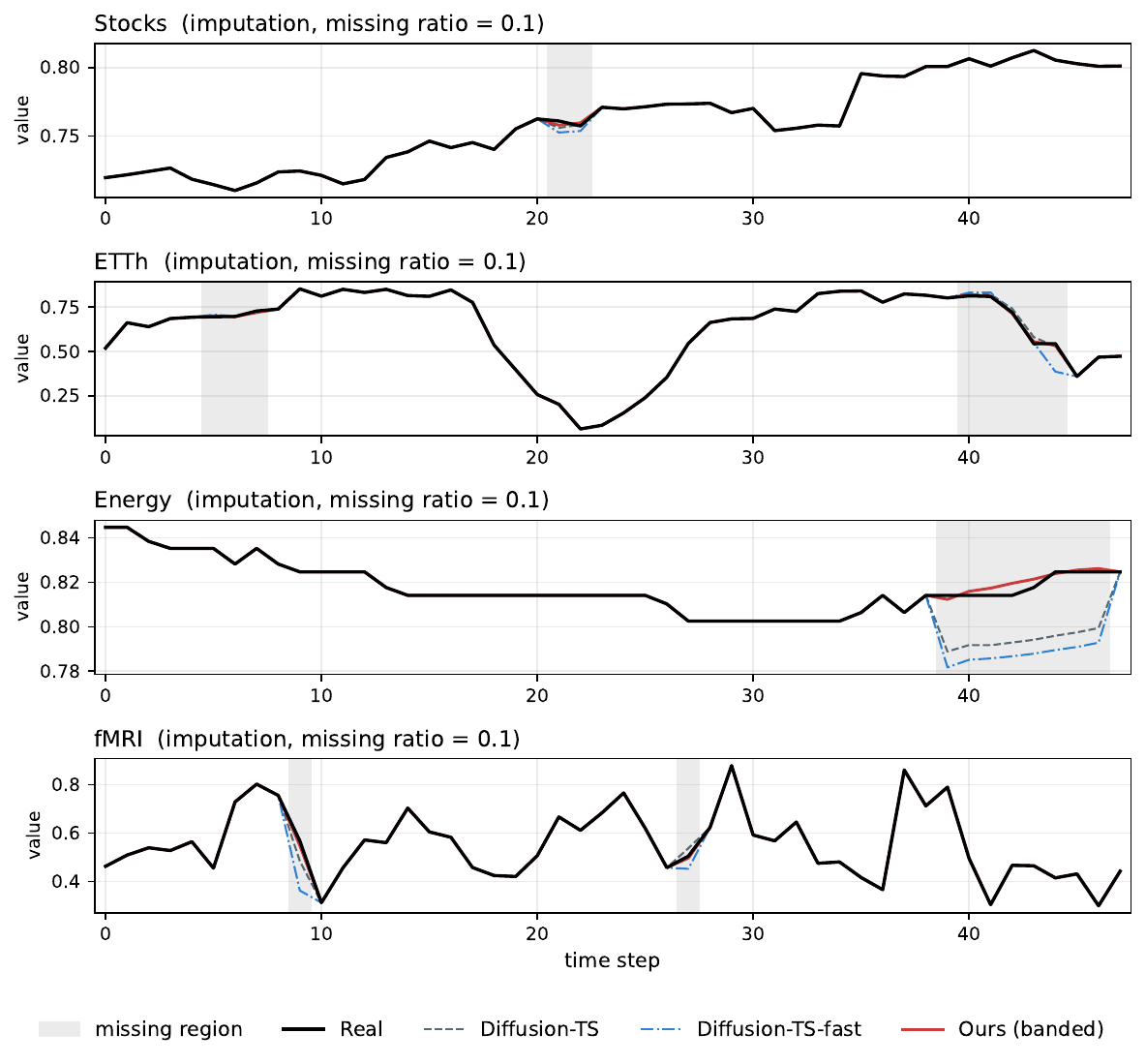}
\caption{Imputation, missing ratio $= 0.1$. Black: ground truth. Light gray shading: held-out positions. Gray dashed: Diffusion-TS DDPM. Blue dash-dot: Diffusion-TS-fast (DDIM-200). Red solid: our band-aware sampler.}
\label{fig:cond_qual_impute_mr0.1}
\end{figure}

\begin{figure}[!h]
\centering
\includegraphics[width=\linewidth]{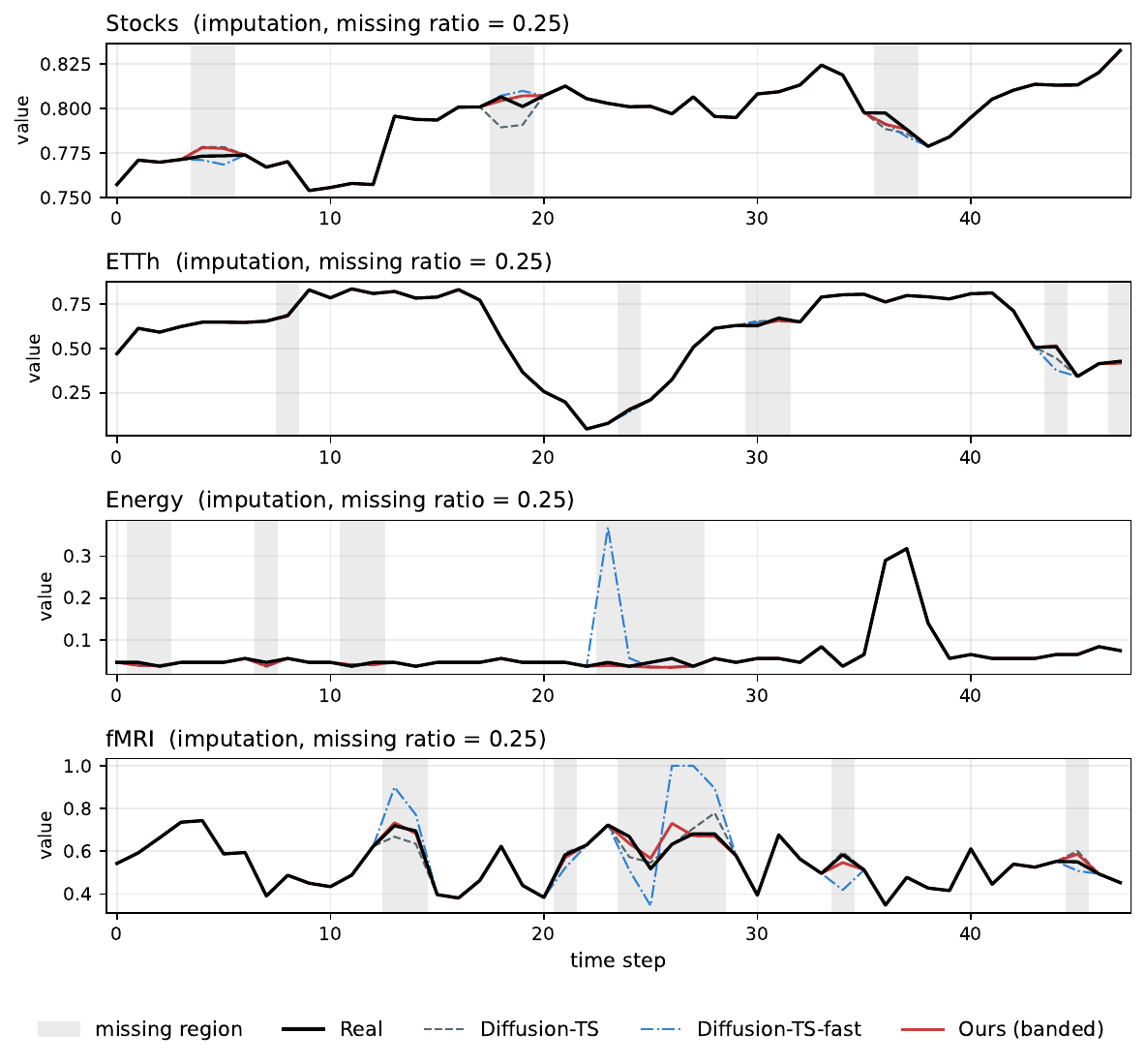}
\caption{Imputation, missing ratio $= 0.25$. Same layout as Fig.~\ref{fig:cond_qual_impute_mr0.1}.}
\label{fig:cond_qual_impute_mr0.25}
\end{figure}

\begin{figure}[!h]
\centering
\includegraphics[width=\linewidth]{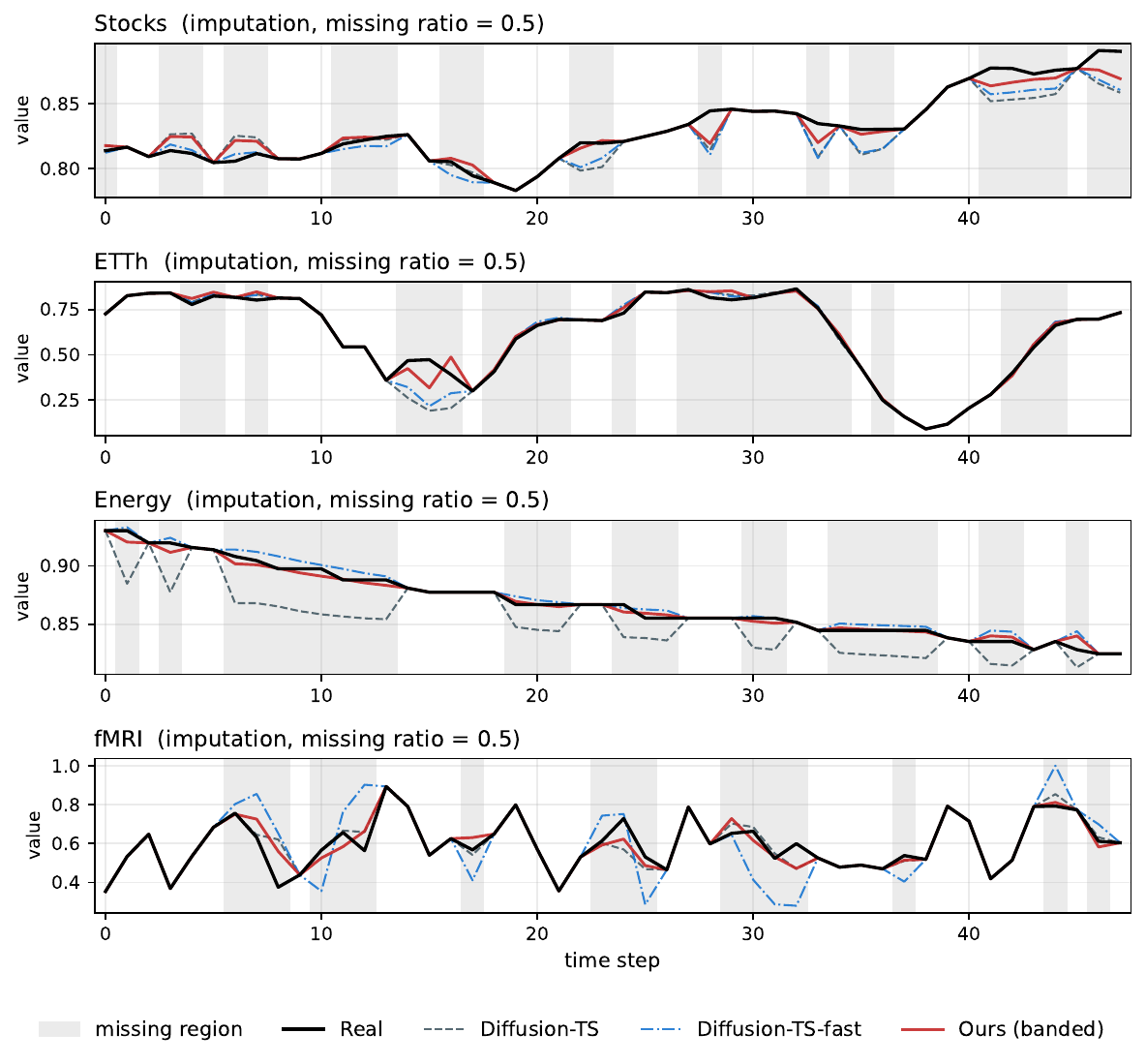}
\caption{Imputation, missing ratio $= 0.5$. Same layout as Fig.~\ref{fig:cond_qual_impute_mr0.1}.}
\label{fig:cond_qual_impute_mr0.5}
\end{figure}

\begin{figure}[!h]
\centering
\includegraphics[width=\linewidth]{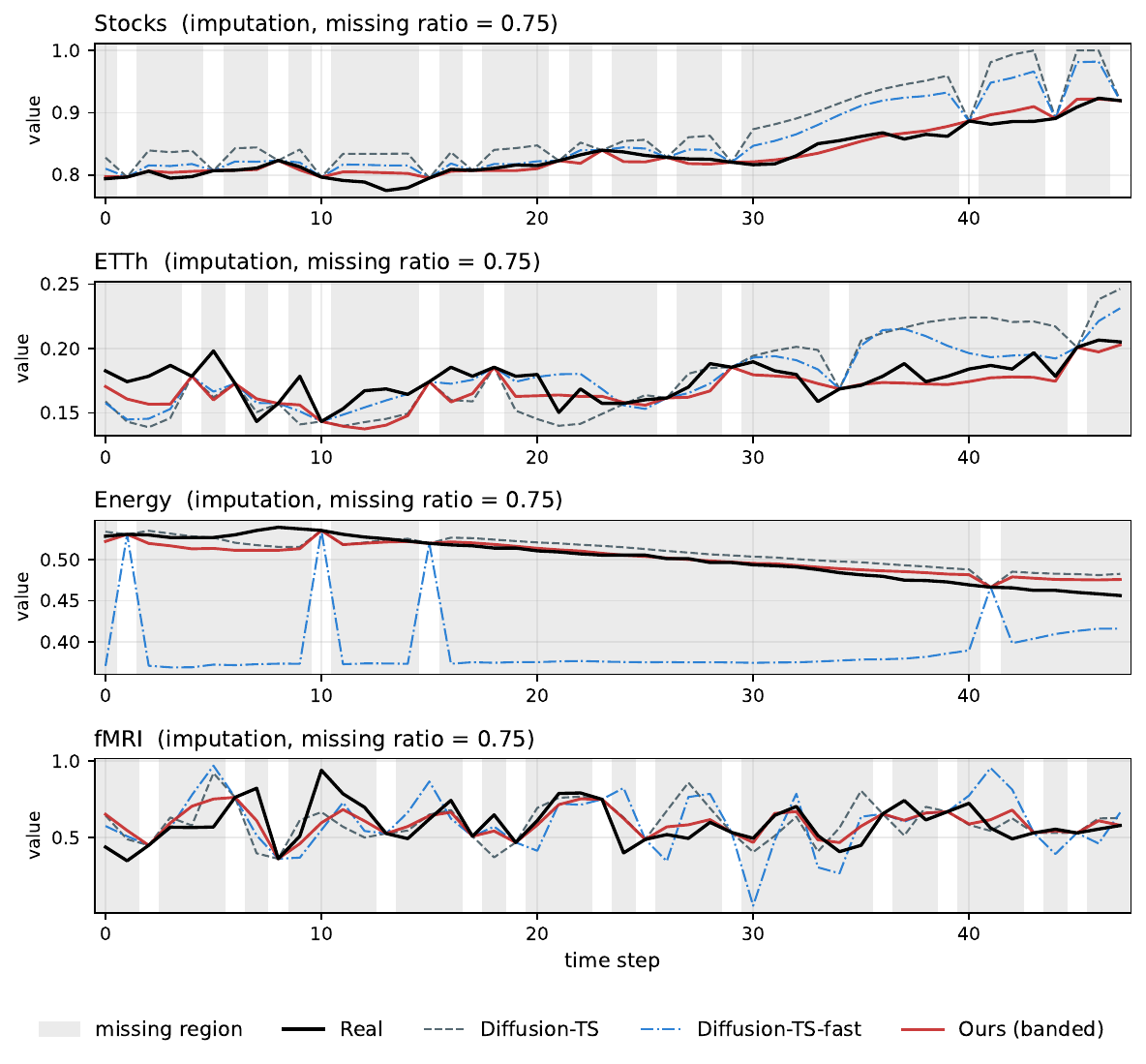}
\caption{Imputation, missing ratio $= 0.75$. Same layout as Fig.~\ref{fig:cond_qual_impute_mr0.1}.}
\label{fig:cond_qual_impute_mr0.75}
\end{figure}

\begin{figure}[!h]
\centering
\includegraphics[width=\linewidth]{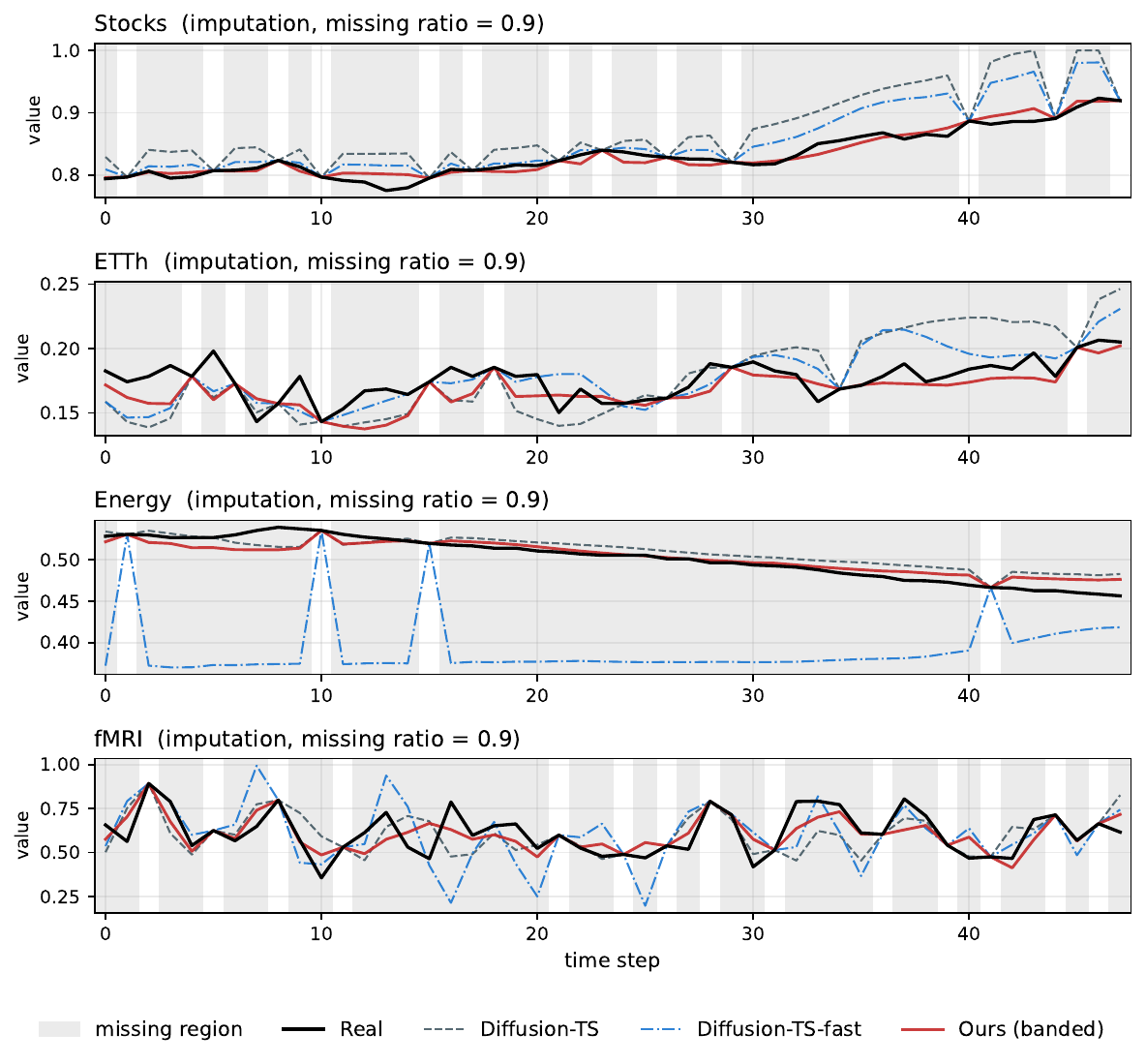}
\caption{Imputation, missing ratio $= 0.9$. Same layout as Fig.~\ref{fig:cond_qual_impute_mr0.1}.}
\label{fig:cond_qual_impute_mr0.9}
\end{figure}

\begin{figure}[!h]
\centering
\includegraphics[width=\linewidth]{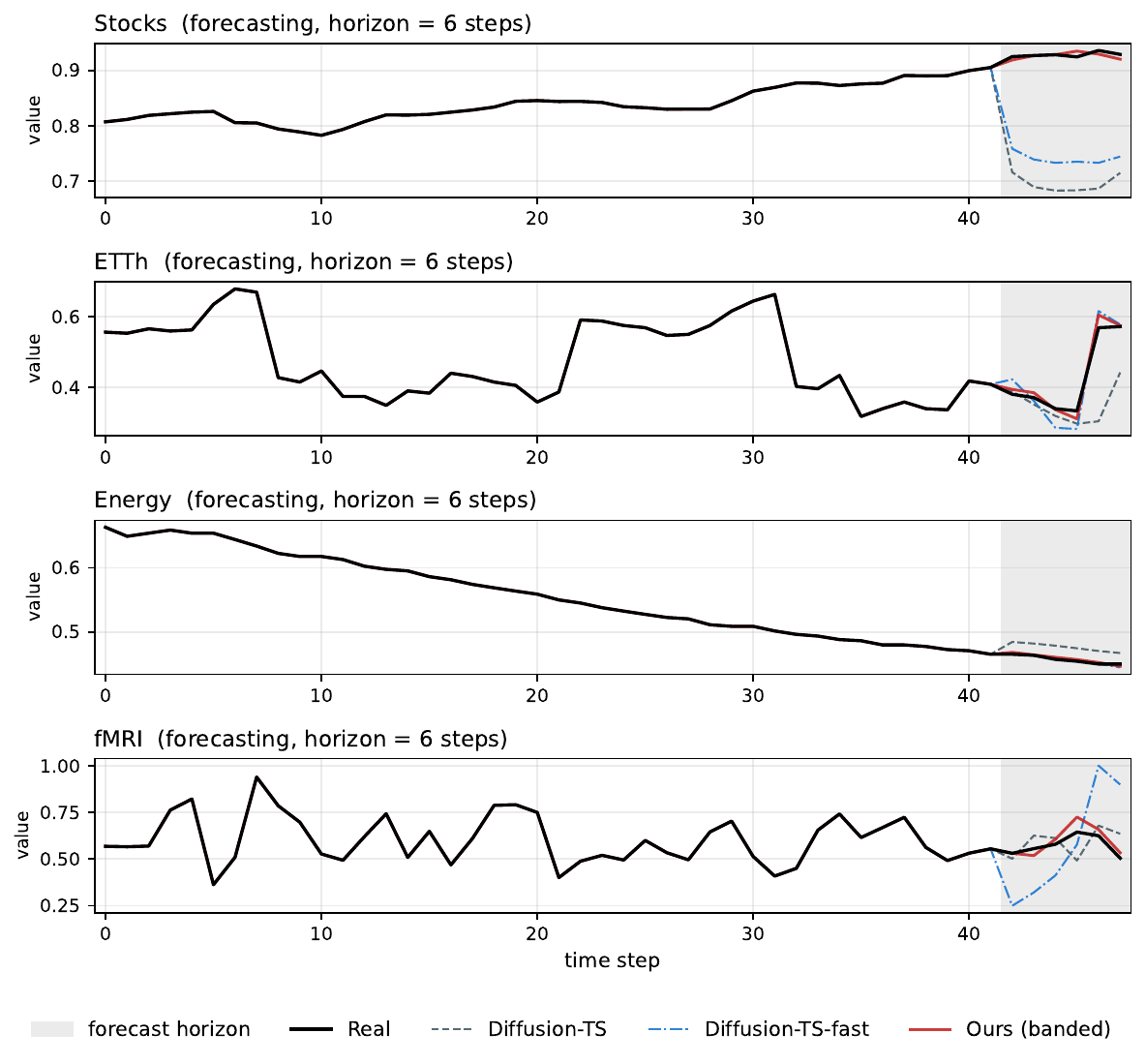}
\caption{Forecasting, prediction horizon $= 6$ steps. Same layout as Fig.~\ref{fig:cond_qual_impute_mr0.1}; the shaded region is the unobserved horizon.}
\label{fig:cond_qual_forecast_pl6}
\end{figure}

\begin{figure}[!h]
\centering
\includegraphics[width=\linewidth]{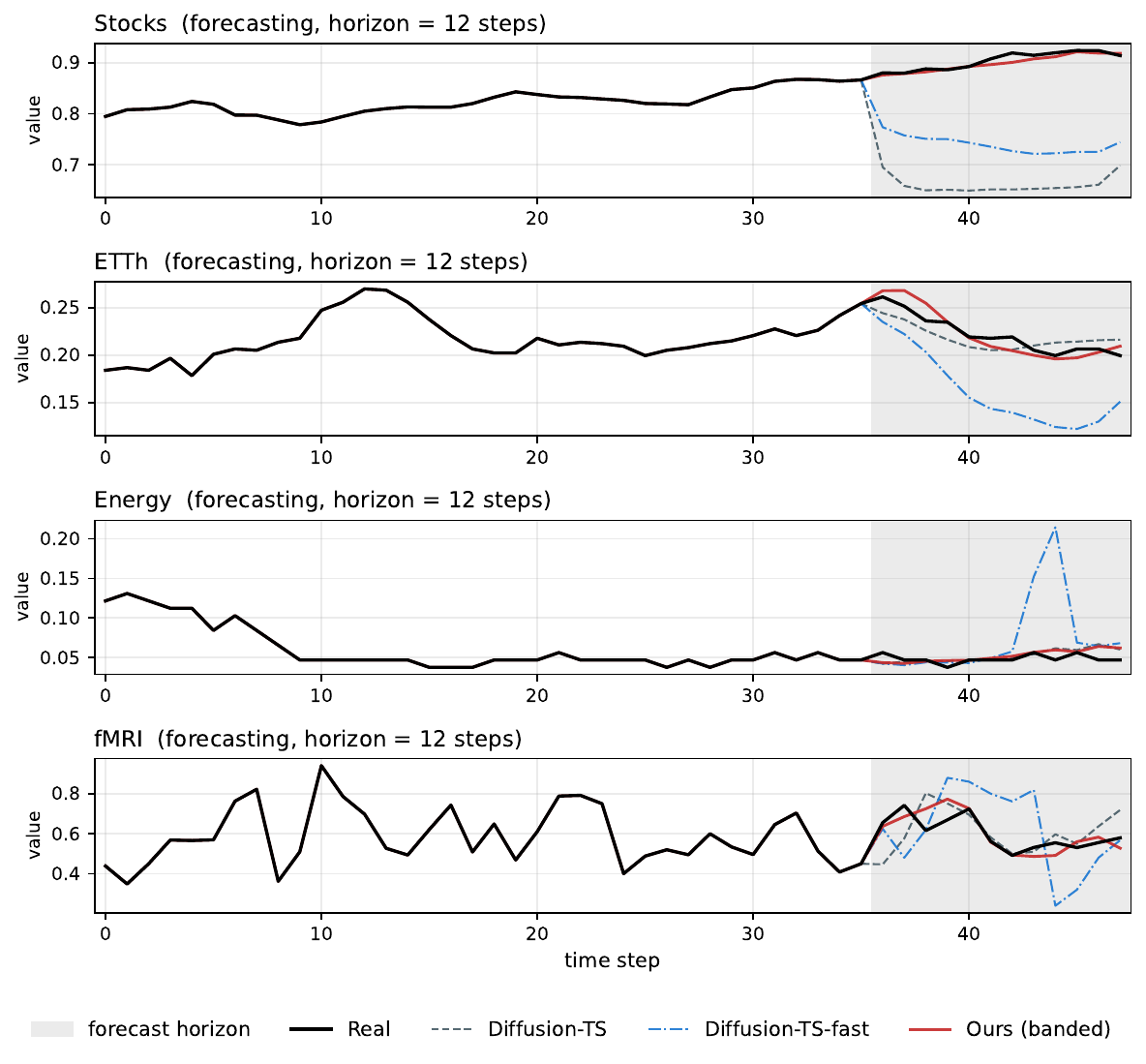}
\caption{Forecasting, prediction horizon $= 12$ steps. Same layout as Fig.~\ref{fig:cond_qual_forecast_pl6}.}
\label{fig:cond_qual_forecast_pl12}
\end{figure}

\begin{figure}[!h]
\centering
\includegraphics[width=\linewidth]{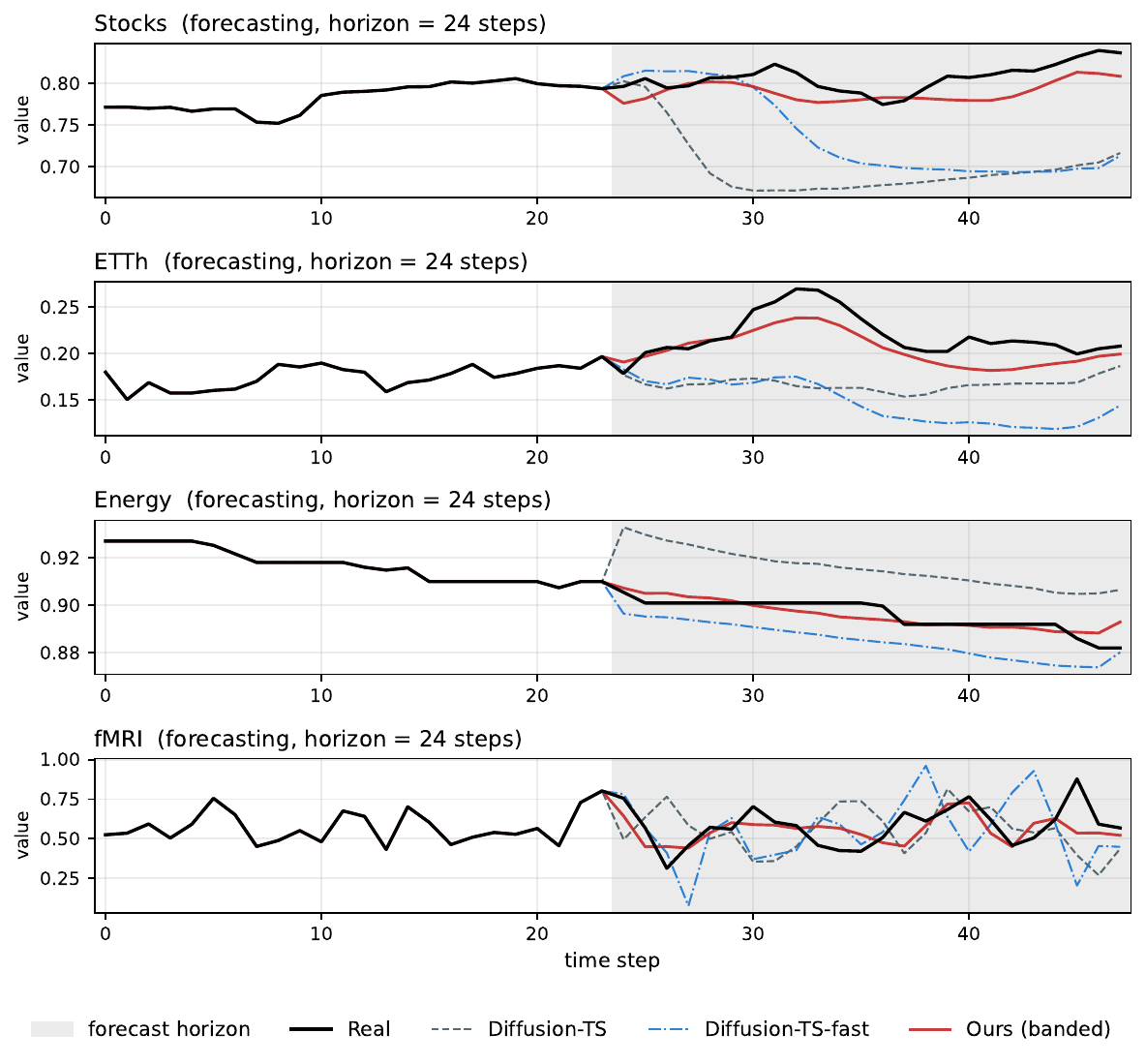}
\caption{Forecasting, prediction horizon $= 24$ steps. Same layout as Fig.~\ref{fig:cond_qual_forecast_pl6}.}
\label{fig:cond_qual_forecast_pl24}
\end{figure}

\begin{figure}[!h]
\centering
\includegraphics[width=\linewidth]{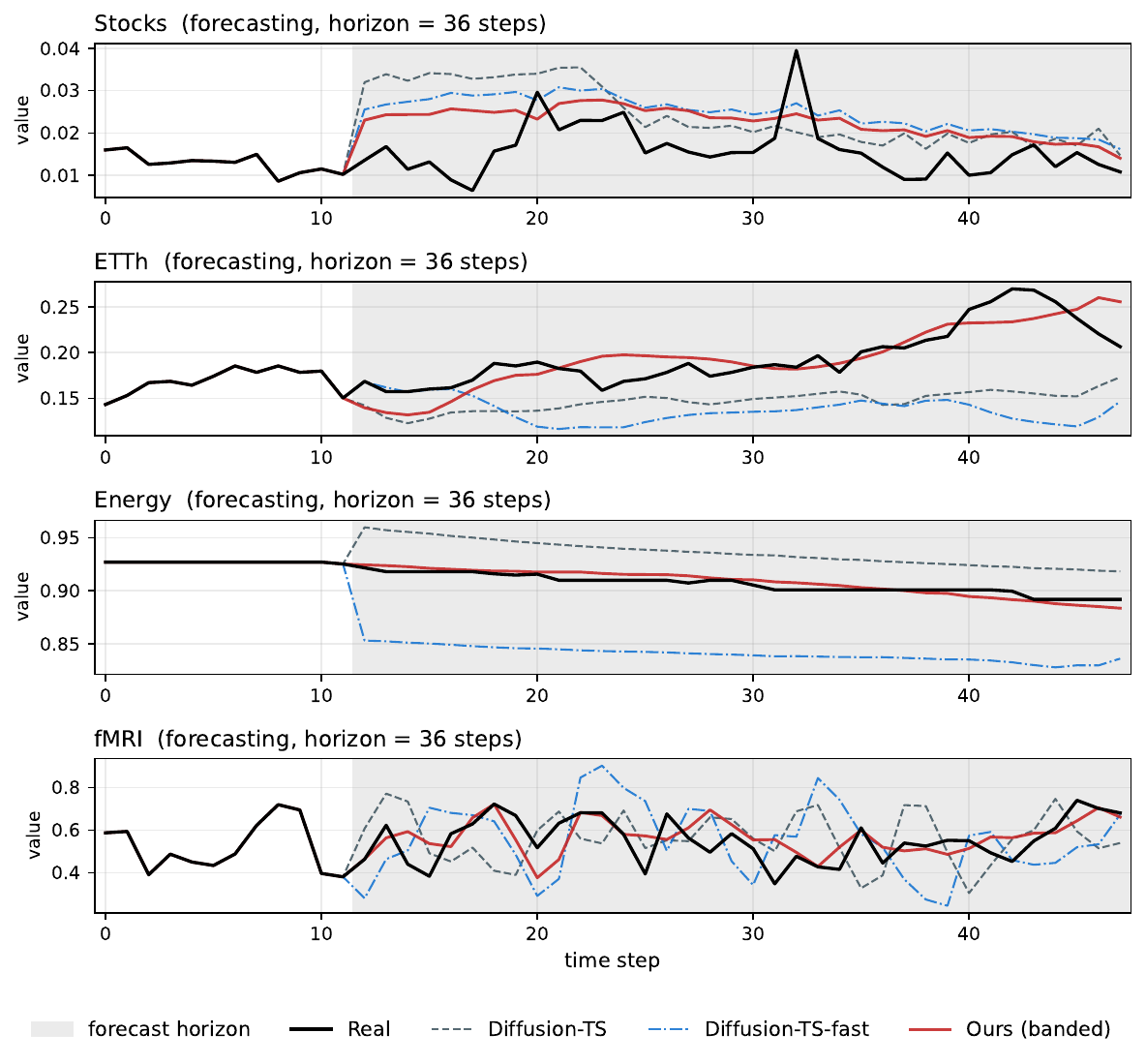}
\caption{Forecasting, prediction horizon $= 36$ steps. Same layout as Fig.~\ref{fig:cond_qual_forecast_pl6}.}
\label{fig:cond_qual_forecast_pl36}
\end{figure}

\section{Visual Comparison of Conditional Generation}
\label{sec:cond_qual_appendix}

This appendix extends the single-dataset case study of
Sec.~\ref{sec:case_study_qual} to all four conditional datasets
(\textsc{stocks}, \textsc{etth}, \textsc{energy}, \textsc{fmri}) and to
the full set of operating points reported in
Tables~\ref{tab:cond_impute} and~\ref{tab:cond_forecast}: imputation at
missing ratios $\{0.1, 0.25, 0.5, 0.75, 0.9\}$
(Figs.~\ref{fig:cond_qual_impute_mr0.1}-\ref{fig:cond_qual_impute_mr0.9})
and forecasting at horizons $\{6, 12, 24, 36\}$
(Figs.~\ref{fig:cond_qual_forecast_pl6}-\ref{fig:cond_qual_forecast_pl36}).
For each (dataset, task) cell, the visualised sample / channel pair is
chosen automatically to maximise horizon variance plus banded
competitiveness (see \texttt{Experiments/plot\_cond\_qualitative.py}); the
selection rule is fixed across methods so the same time series is shown
for all three samplers.

\section{Limitations}
\label{sec:limitations}
Despite our StrideDiffusion achieves great efforts both the acceleration and balance between efficiency and effectiveness. We still have several limitations exist indicating further investigation.

The adaptive leap schedule (Section 4.2.2) heuristically maps the gating results to several predefined step sizes. Sensitivity analysis in Section 5.5 shows that the design is robust to $\tau_{\mathrm{phase}}$ and exhibits only slight sensitivity to the other thresholds within the recommended range; therefore, the heuristic scheduling is not a fragile component. Thus, one possible extension of this paper is to replace the heuristic with a lightweight policy. Using the spectral statistics (relative band energy, log-power shift, phase velocity) already generated by the gating itself as input, a lightweight model is learned to output appropriate step sizes. This extension is intentionally omitted to preserve the design premise that StrideDiffusion is completely training-free during deployment.

\clearpage
\section*{NeurIPS Paper Checklist}

\begin{enumerate}

\item {\bf Claims}
    \item[] Question: Do the main claims made in the abstract and introduction accurately reflect the paper's contributions and scope?
    \item[] Answer: \answerYes{} 
    \item[] Justification: The abstract and introduction state the three concrete contributions.


\item {\bf Limitations}
    \item[] Question: Does the paper discuss the limitations of the work performed by the authors?
    \item[] Answer: \answerYes{} 
    \item[] Justification: We have illustrated the limitation in Appendix~\ref{sec:limitations}

\item {\bf Theory assumptions and proofs}
    \item[] Question: For each theoretical result, does the paper provide the full set of assumptions and a complete (and correct) proof?
    \item[] Answer: \answerYes{} 
    \item[] Justification: Theoretical statements (Lemma~\ref{lem:bandwise_ddim_sensitivity}, Corollary~\ref{cor:inactive_band_stability}, Proposition~\ref{prop:spectral_activity_indicator}, Theorem~\ref{thm:spectral_guided_leap_error}) are stated in Section~\ref{sec:theory_spectral_stability} with their assumptions explicit, and complete proofs are provided in Appendices~\ref{app:proof_bandwise_ddim_sensitivity},~\ref{app:proof_inactive_band_stability},~\ref{app:discussion_xrho_x0},~\ref{app:proof_spectral_activity_indicator}, and~\ref{app:proof_spectral_guided_leap_error}. Theoretical statement of extension to DPM-Solver-2 (Proposition~\ref{prop:dpm_solver_history}) and its proof are stated in Appendix~\ref{app:dpm_solver_extension}.  


    \item {\bf Experimental result reproducibility}
    \item[] Question: Does the paper fully disclose all the information needed to reproduce the main experimental results of the paper to the extent that it affects the main claims and/or conclusions of the paper (regardless of whether the code and data are provided or not)?
    \item[] Answer: \answerYes{} 
    \item[] Justification: Section~\ref{sec:exp_setup} specifies datasets, sequence length, baselines, metrics, and hardware; Section~\ref{sec:adaptive_sample} and the pseudocode in Appendix~\ref{app:inference_pseudocode} fully describe the sampler; per-dataset hyperparameters are listed in the appendix, and an anonymized code repository is linked in the abstract.

\item {\bf Open access to data and code}
    \item[] Question: Does the paper provide open access to the data and code, with sufficient instructions to faithfully reproduce the main experimental results, as described in supplemental material?
    \item[] Answer: \answerYes{} 
    \item[] Justification: An anonymized code repository is provided in the abstract (\url{https://anonymous.4open.science/r/stridediff-ts}), containing the sampler implementation, training/inference scripts, and configuration files. All datasets used (Sines, Stocks, ETTh, MuJoCo, Energy, fMRI) are publicly available and follow the Diffusion-TS preprocessing protocol referenced in Section~\ref{sec:exp_setup}.

\item {\bf Experimental setting/details}
    \item[] Question: Does the paper specify all the training and test details (e.g., data splits, hyperparameters, how they were chosen, type of optimizer) necessary to understand the results?
    \item[] Answer: \answerYes{} 
    \item[] Justification: Section~\ref{sec:exp_setup} reports baselines, datasets, evaluation metrics, sequence length, and the computing setup. Hyperparameters of the sampler ($l_{\mathrm{coarse}}$, $l_{\mathrm{mid}}$, $l_{\mathrm{fine}}$, $K_{\mathrm{micro}}$, $\tau_{\mathrm{energy}}$, $\tau_{\mathrm{phase}}$) and how they are selected are described in Section~\ref{sec:adaptive_sample} and analyzed in Section~\ref{sec:sensitivity} and Appendix~\ref{sec:sensitivity_appendix}. Since StrideDiffusion is training-free, it inherits the pretrained Diffusion-TS backbone whose training protocol we follow without modification.

\item {\bf Experiment statistical significance}
    \item[] Question: Does the paper report error bars suitably and correctly defined or other appropriate information about the statistical significance of the experiments?
    \item[] Answer: \answerYes{} 
    \item[] Justification: All quantitative tables in the unconditional generation experiments (Table~\ref{tab:results-unconditional-generation}) report mean$\pm$standard deviation across multiple random seeds, capturing run-to-run variability for both fidelity and diversity metrics. Conditional generation tables report means across several seeds.

\item {\bf Experiments compute resources}
    \item[] Question: For each experiment, does the paper provide sufficient information on the computer resources (type of compute workers, memory, time of execution) needed to reproduce the experiments?
    \item[] Answer: \answerYes{} 
    \item[] Justification: Section~\ref{sec:exp_setup} reports the hardware (24 vCPUs / 48 threads Intel Xeon Silver 4310 @ 2.10\,GHz, 256\,GiB RAM, single NVIDIA A5000 GPU, 24\,GB VRAM) on which all timings are measured, and Tables~\ref{tab:results-unconditional-generation} and~\ref{tab:cond_impute} report wall-clock inference time per dataset and configuration.
    
\item {\bf Code of ethics}
    \item[] Question: Does the research conducted in the paper conform, in every respect, with the NeurIPS Code of Ethics \url{https://neurips.cc/public/EthicsGuidelines}?
    \item[] Answer: \answerYes{} 
    \item[] Justification: The research uses only publicly available, non-sensitive time-series benchmarks and does not involve human subjects, personally identifiable information, or deployed decision-making systems. The work conforms with the NeurIPS Code of Ethics.

\item {\bf Broader impacts}
    \item[] Question: Does the paper discuss both potential positive societal impacts and negative societal impacts of the work performed?
    \item[] Answer: \answerNA{} 
    \item[] Justification: The paper proposes a training-free sampler that accelerates inference of pretrained time-series diffusion models; it neither introduces a new generative capability nor targets sensitive applications. The contribution is foundational and methodological, with no direct path to negative societal applications beyond those already attached to time-series generation in general.
    
\item {\bf Safeguards}
    \item[] Question: Does the paper describe safeguards that have been put in place for responsible release of data or models that have a high risk for misuse (e.g., pre-trained language models, image generators, or scraped datasets)?
    \item[] Answer: \answerNA{} 
    \item[] Justification: We release only sampler code; no pretrained generative model, scraped dataset, or other high-risk asset is released. The released code uses publicly available time-series benchmarks under their original licenses.

\item {\bf Licenses for existing assets}
    \item[] Question: Are the creators or original owners of assets (e.g., code, data, models), used in the paper, properly credited and are the license and terms of use explicitly mentioned and properly respected?
    \item[] Answer: \answerYes{} 
    \item[] Justification: All baselines (Diffusion-TS, DiffWave, DiffTime/CSDI, DDIM, DPM-Solver) and benchmark datasets (Sines, Stocks, ETTh, MuJoCo, Energy, fMRI) are cited in Sections~\ref{sec:related_work} and~\ref{sec:exp_setup}. We use them under their original open-source licenses (predominantly MIT/Apache-2.0 for code and the dataset providers' standard terms) and follow each dataset's preprocessing protocol as specified in the Diffusion-TS reference implementation.

\item {\bf New assets}
    \item[] Question: Are new assets introduced in the paper well documented and is the documentation provided alongside the assets?
    \item[] Answer: \answerYes{} 
    \item[] Justification: The only new asset is the StrideDiffusion sampler implementation, released anonymously at \url{https://anonymous.4open.science/r/stridediff-ts} with a README, configuration files, and scripts to reproduce the main tables. Pseudocode is also provided in Appendix~\ref{app:inference_pseudocode}.

\item {\bf Crowdsourcing and research with human subjects}
    \item[] Question: For crowdsourcing experiments and research with human subjects, does the paper include the full text of instructions given to participants and screenshots, if applicable, as well as details about compensation (if any)? 
    \item[] Answer: \answerNA{} 
    \item[] Justification: The paper does not involve crowdsourcing or research with human subjects; all experiments use publicly available time-series benchmarks.

\item {\bf Institutional review board (IRB) approvals or equivalent for research with human subjects}
    \item[] Question: Does the paper describe potential risks incurred by study participants, whether such risks were disclosed to the subjects, and whether Institutional Review Board (IRB) approvals (or an equivalent approval/review based on the requirements of your country or institution) were obtained?
    \item[] Answer: \answerNA{} 
    \item[] Justification: The paper does not involve human subjects, so IRB approval is not applicable.

\item {\bf Declaration of LLM usage}
    \item[] Question: Does the paper describe the usage of LLMs if it is an important, original, or non-standard component of the core methods in this research? Note that if the LLM is used only for writing, editing, or formatting purposes and does \emph{not} impact the core methodology, scientific rigor, or originality of the research, declaration is not required.
    \item[] Answer: \answerNA{} 
    \item[] Justification: LLMs are not used as a component of the proposed method; the sampler operates on pretrained time-series diffusion backbones and uses no language-model component. Any LLM use was limited to writing/editing assistance, which does not require declaration.

\end{enumerate}

\end{document}